\documentclass[twoside]{article}
\usepackage[preprint]{aistats2024}
\usepackage{natbib}
\usepackage{empheq}

\usepackage{xcolor}         
\definecolor{indigo}{rgb}{0.2, 0.0, 0.45}

\usepackage[utf8]{inputenc} 
\usepackage[T1]{fontenc}    
\usepackage[colorlinks=true, allcolors=indigo]{hyperref}       
\usepackage{url}            
\usepackage{booktabs}       
\usepackage{amsfonts}       
\usepackage{nicefrac}       
\usepackage{microtype}      
\usepackage{adjustbox}
\usepackage{pst-all}
\usepackage{multirow}
\usepackage{comment}
\usepackage{caption}
\usepackage{todonotes}
\usepackage{amsmath}
\usepackage{amssymb}
\usepackage{mathtools}
\usepackage{amsthm}
\usepackage{cleveref}
\usepackage{longtable}
\usepackage{graphicx}
\usepackage{subcaption}
\usepackage{tabularx}
\usepackage{stmaryrd}
\usepackage{array}
\usepackage{xr}
\usepackage{bm}
\usepackage{color, colortbl}
\usepackage{dashrule}

\usepackage{tikz}
\usepackage{colortbl} 

\newcommand{\indep}{\rotatebox[origin=c]{90}{$\models$}}

\newcommand{\stylizedX}{
    \begin{tikzpicture}[baseline=-0.5ex]
        \draw[red, thick, rotate=45] (-0.15,0) -- (0.15,0);
        \draw[red, thick, rotate=-45] (-0.15,0) -- (0.15,0);
    \end{tikzpicture}
}

\newcommand{\stylizedPlus}{
    \begin{tikzpicture}[baseline=-0.5ex]
        \draw[green, thick] (0,-0.15) -- (0,0.15);
        \draw[green, thick] (-0.15,0) -- (0.15,0);
    \end{tikzpicture}
}

\usepackage[normalem]{ulem}

\usepackage{algorithm}      
\usepackage{algpseudocode}  

\definecolor{yellow}{rgb}{0.91, 0.84, 0.42}
\definecolor{blue}{rgb}{0.6, 0.73, 0.89}
\definecolor{red}{rgb}{0.8, 0.25, 0.33}
\definecolor{green}{rgb}{0.7, 0.75, 0.71}
\definecolor{orange}{rgb}{1.0, 0.6, 0.4}
\definecolor{cream}{rgb}{1.0, 0.99, 0.82}
\definecolor{Gray}{gray}{0.9}
\definecolor{green}{rgb}{0.6, 0.73, 0.89}
\definecolor{lightblue}{RGB}{225, 240, 255}

\DeclareMathOperator*{\argmax}{arg\,max}
\DeclareMathOperator*{\argmin}{arg\,min}
\DeclareMathOperator*{\argkmax}{argkmax}
\DeclareMathOperator*{\supp}{supp}

\newcommand{\cut}[1]{}

\theoremstyle{plain}
\newtheorem{theorem}{Theorem}[section]

\theoremstyle{definition}

\newtheorem{lemma}[theorem]{Lemma}
\newtheorem{corollary}[theorem]{Corollary}
\newtheorem{definition}[theorem]{Definition}

\theoremstyle{remark}

\author{%
  William Toner\\
  University of Edinburgh\\
  \texttt{w.j.toner@sms.ed.ac.uk} \\
   \And
   Amos Storkey \\
   University of Edinburgh\\
   \texttt{a.storkey@ed.ac.uk} \\
}
\begin{document}
\pagestyle{plain}
\twocolumn[
\aistatstitle{Noisy Early Stopping for Noisy Labels}
\aistatsauthor{ William Toner
\And Amos Storkey}
\aistatsaddress{University of Edinburgh \\
  \texttt{w.j.toner@sms.ed.ac.uk}  \And  University of Edinburgh\\
   \texttt{a.storkey@ed.ac.uk}}]
\begin{abstract}
Training neural network classifiers on datasets contaminated with noisy labels significantly increases the risk of overfitting. Thus, effectively implementing Early Stopping in noisy label environments is crucial. Under ideal circumstances, Early Stopping utilises a validation set uncorrupted by label noise to effectively monitor generalisation during training. However, obtaining a noise-free validation dataset can be costly and challenging to obtain. This study establishes that, in many typical learning environments, a noise-free validation set is not necessary for effective Early Stopping. Instead, near-optimal results can be achieved by monitoring accuracy on a \emph{noisy} dataset - drawn from the same distribution as the noisy training set. Referred to as `Noisy Early Stopping` (NES), this method simplifies and reduces the cost of implementing Early Stopping. We provide theoretical insights into the conditions under which this method is effective and empirically demonstrate its robust performance across standard benchmarks using common loss functions.
\end{abstract}
\vspace{-2mm}


\section{Introduction}

\subsection{Context and Problem Statement}
In recent years, the field of machine learning has experienced significant advancements in classification accuracy, driving demand for large, well-labelled datasets to train increasingly sophisticated models. However, the cost of meticulously curating large-scale, clean datasets is high, pushing practitioners toward alternative sources that, while abundant, are rife with label inaccuracies. Common methods such as crowd-sourcing or automated web scraping frequently yield data contaminated by errors. Even traditional methods of data collection are not immune, often suffering from inaccuracies due to human error, particularly in fields requiring expert knowledge for accurate labelling, like astronomy or medical diagnostics. The prevalence of inaccurately labelled data has heightened the need for machine learning algorithms that can effectively navigate these challenges. This has led to a growing focus on developing methods that can robustly handle noisy labels, which is becoming a key area of interest in the machine learning community.

A variety of techniques have been developed to address the issue of label noise in datasets. These techniques include label correction approaches, data cleansing processes, regularisation methods, and data augmentation strategies. While these methods often prove effective, they can also be computationally demanding and intricate. Some require the use of multiple neural networks or entail complex combinations of various techniques \citep{kim2024learning, dividemix, coteaching, mentornet, decoupling, meta_dynamic, meta_gradient,EvidenceMix}. Such complexity can make these methods less practical for use in environments where time, expertise, or computational resources are limited.

\textbf{Robust loss functions} have been developed as a simple and inexpensive approach for handling label noise. These approaches are designed to counteract the tendency of the widely used cross-entropy loss to overfit when faced with noisy data.  While these functions are designed to enhance training robustness, they often still result in overfitting or underfitting in various settings \citep{GCE_Loss, normalised_losses}. It is unrealistic to expect that a single loss function could perfectly avoid both overfitting and underfitting across all datasets, noise models, and hyperparameter configurations. Therefore, it becomes crucial to understand how and when to implement Early Stopping, particularly when cleanly labelled validation datasets are not available to monitor generalisation during training.


\paragraph*{Scope} This paper focuses on preventing the overfitting of \textbf{neural network classifiers} trained with \textbf{robust loss functions} on datasets affected by \textbf{label noise}.

\paragraph*{Contributions} The primary aim of this study is to demonstrate that noisy validation accuracy—i.e., accuracy on a held-out dataset, drawn from the \uline{same distribution as the noisy training set}—can reliably predict generalisation to \emph{clean} (noise-free) data distributions. Consequently, noisy accuracy can be used to define an effective policy for Early Stopping. This involves ceasing training when the noisy validation accuracy begins to decline, a strategy we term `Noisy Early Stopping'.

Our findings are consistent across standard image datasets and typical noise models, including non-uniform and asymmetric label noise. These insights are particularly valuable as they: 1) Provide ML practitioners with a simple and reliable way to early stop in the presence of label noise, and 2) Validate the common practice of using noisy test accuracy to evaluate and compare label noise robust algorithms in the absence of cleanly labelled datasets (e.g. for the Clothing1M dataset, which lacks a cleanly labelled subset).




\subsection{Paper Outline}
In Section~\ref{ch7:sec:background} introduces the relevant notation and terminology. We define `Noisy Early Stopping' (NES) as the strategy of Early Stopping by monitoring performance on a validation set polluted by label noise. Section~\ref{ch7:sec:related_work}, discusses related work. In Section~\ref{ch7:sec:theoretical}, we derive relationships between the clean and noisy $0\text{-}1$ risks of a model in different label noise environments.

 The theory suggests that NES should be effective for symmetric label noise and not for other noise types. In Section~\ref{ch7:sec:experiments}, we empirically evaluate NES. Remarkably we demonstrate the effectiveness of NES across various datasets, noise models, and six popular robust loss functions. 

\section{Background}\label{ch7:sec:background}

\subsection{Notation and Terminology}
\paragraph*{Domains} $\mathcal{X} \subset \mathbb{R}^d$ represents the dataspace, and $\mathcal{Y} \coloneqq \{1,2,\ldots, c\}$ denotes the label space with $c$ the total number of classes. The probability simplex, denoted by $\Delta$, consists of $c$-dimensional vectors whose non-negative components sum to one. Vector quantities are presented in \textbf{bold}.

\paragraph*{Probability Estimator} A probability estimator (often abbreviated to just `estimator') is a model that assigns a distribution over classes to each location in the data space, represented as $\bm{q}:\mathcal{X}\rightarrow \Delta$.

\paragraph*{Loss Function} A loss function $L: \Delta \times \mathcal{Y} \rightarrow \mathbb{R}$ measures the discrepancy between predicted and actual label distributions, resulting in a loss value. 

An important loss function is the $0\text{-}1$ loss which outputs $0$ when a model predicts the correct class and $1$ otherwise:
\[
L_{0\text{-}1}(\bm{q}, k) = 
\begin{cases} 
0 & \text{if } \arg\max_i(q_i) = k, \\
1 & \text{if } \arg\max_i(q_i) \neq k,
\end{cases}
\]

We say that a loss function $L$ is \emph{Fisher consistent} if a minimiser of the expected loss under $L$ incurs minimal expected $0\text{-}1$ error.

\paragraph*{$L$-Risk} The (generalised) $L$-risk, representing the expected loss over the \textit{entire} data distribution $p(x,y)$, evaluates the overall efficacy of estimator $\bm{q}(x)$:
\begin{align*}
    R_{L}(\bm{q}) \coloneqq \mathbb{E}_{x,y \sim p(x,y)}[L(\bm{q}(x),y) ].
\end{align*}

We refer to the risk computed from an i.i.d. dataset drawn from $p(x,y)$, the \emph{empirical} risk which we denote $\widehat{R}_L$.

\textbf{Unless specified otherwise, the term \textit{`risk'} throughout this document refers to the \(0\text{-}1\) risk.} 

\paragraph{Label Noise} Label noise refers to any random process that modifies labels drawn from the data-label distribution $p(x,y)$. We assume that this alteration is characterised by a conditional distribution $p(\widetilde{y}\mid x,y)$ which describes the transformation of clean labels into noisy ones. To differentiate between the original (clean) and altered (noisy) labels, we use a tilde notation, such as $\widetilde{p}(x,\widetilde{y})$. A detailed taxonomy of label noise types is provided in the Appendix.

\paragraph{Noisy Risk} Label noise gives rise to the concept of \emph{noisy} generalised risk, denoted by \( R^{\eta}_{L} \), computed with respect to the noisy label distribution.

\paragraph{`Noisy' and `Clean'} 
We adopt the practice of using `\emph{noisy}' and `\emph{clean}' to distinguish between quantities evaluated with respect to the noised or un-noised data distributions respectively. For example, `clean accuracy' refers to accuracy computed on samples drawn from the un-noised distribution whereas `noisy accuracy' refers to accuracy computed on samples drawn from the noisy distribution $\widetilde{p}(x,\widetilde{y})$ 
Similarly, `clean Early Stopping' refers to Early Stopping where generalisation is monitored using a cleanly-labelled validation set, this is to be contrasted with `Noisy Early Stopping' where generalisation is evaluated on a validation set whose labels are corrupted with label noise. 

\begin{figure}[!ht]
\centering
  \includegraphics[width=\linewidth]{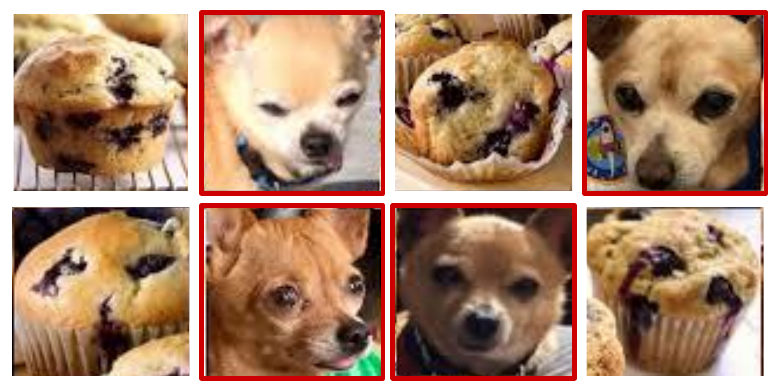}
  \caption{Illustration of the difference between \emph{noisy} and \emph{clean} accuracy and the non-trivial relationship between them: The figure depicts eight images from a web-scraped chihuahua dataset. The dataset contains label noise as it has accidentally scraped images of muffins \citep{muffin_chihuahua_dataset}. A classifier, which correctly identifies all of the true chihuahuas (red), will obtain a \textbf{noisy accuracy} of $\frac{4}{8}\approx 50\%$ on this dataset despite obtaining a \textbf{clean accuracy} of $100\%$.}
  \label{fig:clean_noisy_accuracy}
\end{figure}

\subsubsection{Class-Preserving Label Noise}
Most work studying label noise robust learning involves some restriction on the label noise model. Without restrictions, label noise can be so extreme that it obliterates any useful signal for learning: For example, consider a label noise model that randomly permutes 100\% of labels in a dataset. An early example of a restriction applied to the label noise model occurs in the seminal work of \citet{angluin1988} who limit the noise rate to $\eta<1/2$ stating \emph{`Why do we restrict $\eta$ to be less than $1/2$? Clearly, when $\eta = 1/2$, the errors in the reporting process destroy all possible information about
membership in the unknown set and no identification procedure could
be expected to work'}. Similar assumptions are in made in \cite{old_backward, robust_theory, polyPerceptron, AsymmetricLoss} etc.

In this study, we adopt a multi-class generalisation of the assumption used in \citet{angluin1988} which we call the `class-preserving label noise assumption'. This is almost identical to the `clean-labels-domination' assumption defined in \citet{AsymmetricLoss} and describes the notion that the most likely noisy label is the same as the most likely clean label at every location in dataspace.

\begin{definition}[Class-Preserving Label Noise]\label{ch3:def:conserveNoise}
Given a data-label distribution \( p(x,y) \), and its noisy version \( \widetilde{p}(x,\widetilde{y}) \), the label noise is considered \textit{class-preserving} if, for every \( x \in \mathcal{X} \), the class with the highest probability (known as the `dominant class' \citep{hastie2009elements}) remains unchanged after noise application. 
Formally, the condition for class-preserving noise is expressed as:
\[
\argmax_{i \in \{1, 2, \ldots, c\}} \widetilde{p}(\widetilde{y}=i | x) = \argmax_{i \in \{1, 2, \ldots, c\}} p(y=i | x),
\]
where \( c \) is the number of classes.
\end{definition}

Class-preserving noise \textit{preserves} which class has the highest probability. For example, suppose that we have a labelled dataset of images of animals. We may expect some of the wolves to be mislabelled as dogs. If wolves are \emph{more likely} to be mislabelled as dogs than to be correctly labelled as wolves then this noise is \emph{not} class-preserving. Symmetric label noise is class-preserving whenever the noise rate $\eta$ satisfies $\eta < \frac{c-1}{c}$, e.g. less than $90\%$ for $c=10$ classes.

\subsection{Early Stopping}
Early Stopping (ES) is a regularisation technique in machine learning, designed to prevent overfitting by terminating the training process once generalisation begins to worsen \citep{prechelt2002early}. While there are numerous different criteria which can be used to determine when to early-stop, the `gold-standard' approach \citep{early_stop_without_val} is monitoring performance on a held-out validation set, drawn from the distribution of interest. Suppose we are learning a probability estimator, letting $\bm{q}_n$ denote the model obtained after $n$ epochs of training. We have a sequence of models $\bm{q}_1, \bm{q}_2, \bm{q}_3, \ldots$. Utilising a loss function $L$ and a large validation dataset, we can estimate the risks of these models with low variance: $\widehat{R}_{L}(\bm{q}_n) \approx R_{L}(\bm{q}_n)$. This allows us to reliably detect when generalisation starts to decline, marking an optimal point to halt training. The cessation occurs if no improvements are noted over a set number of epochs, a period known as the `patience'. 

\paragraph*{The Problem Of Noise:} Early Stopping is generally effective when the validation set is representative of the distribution we aim to generalise to. However, complications arise when both the training and validation datasets are contaminated by systematic noise. Consider a regression scenario where an agent randomly adds a small positive value $\epsilon \in [0,1]$ to the targets in both the training and validation sets. In this case, the model may inadvertently fit to this noise, impairing its ability to generalise to the true underlying distribution. Crucially, Early Stopping is unlikely to correct for this issue, as the validation set, being similarly polluted, does not accurately reflect the target distribution we seek to generalise to.\footnote{In this study we look exclusively at \emph{classification}, not regression, this example is purely illustrative.} 

\subsection{Main Proposal: Noisy Early Stopping} 
When Early Stopping is applied by monitoring the accuracy on a validation set that is subject to the same label noise as the training set we call this \textbf{Noisy Early Stopping} (NES) (Expressed algorithmically in Algorithm~\ref{alg:noisy_early_stopping}). This technique aims to optimise training cessation to maximise the model's generalisation performance on \emph{cleanly-labelled} data. We contrast Noisy Early Stopping with clean Early Stopping (Abbreviated simply as ES) which is an idealised setting in which we have a cleanly-labelled validation set by which to monitor clean performance. 

\begin{algorithm}
\caption{Noisy Early-Stopping Algorithm}\label{ch7:sec:alg}
\begin{algorithmic}[1]
\Require{Noisy dataset $D$, split into training set $D_{\text{train}}$ and validation set $D_{\text{val}}$; Fisher consistent Loss function $L$ (e.g., cross-entropy); Patience parameter $P$}
\Ensure{Trained model with Early Stopping based on noisy validation accuracy}

\State initialise neural network $f$ with random weights
\State initialise patience counter $p \gets 0$
\State initialise best validation accuracy $\text{acc}_{\text{best}} \gets 0$

\While{$p < P$}
    \For{each epoch}
        \State Train $f$ on $D_{\text{train}}$ using loss function $L$
        \State Calculate noisy accuracy $\text{acc}_{\text{epoch}}$ on $D_{\text{val}}$
        \If{$\text{acc}_{\text{epoch}} > \text{acc}_{\text{best}}$}
            \State Update $\text{acc}_{\text{best}} \gets \text{acc}_{\text{epoch}}$
            \State Save current model as the best model
            \State Reset patience counter $p \gets 0$
        \Else
            \State Increment patience counter $p \gets p + 1$
        \EndIf
    \EndFor
    \If{$p \geq P$}
        \State \textbf{break}
    \EndIf
\EndWhile

\State Load the best saved model
\State \Return{The trained model $f$}
\end{algorithmic}\label{alg:noisy_early_stopping}
\end{algorithm}

\subsection{Assumptions and Problem Statement} 
The primary objective of this work is to evaluate the effectiveness of the Noisy Early Stopping algorithm (Algorithm~\ref{alg:noisy_early_stopping}) across a broad range of conditions when training with noisy labels. We begin with a theoretical analysis in Section~\ref{ch7:sec:theoretical}, based on two idealised assumptions:

\begin{enumerate}
 \item We assume the availability of an arbitrarily large validation dataset, allowing the noisy risk at each epoch to be estimated with high precision and independently\footnote{i.e. The noisy risk estimates are statistically independent random variables; $\widehat{R}^{\eta}(\bm{q}_i)\indep \widehat{R}^{\eta}(\bm{q}_j)$}.
    \item We assume a large patience parameter, allowing Noisy Early Stopping to select from any model obtained during training.
\end{enumerate}

These assumptions allow us to formulate the problem as follows:
\paragraph{Problem Statement:} Consider a data-label distribution \( p(x, y) \) and its noisy-label counterpart \( \widetilde{p}(x, \widetilde{y}) \). Let \( \mathcal{Q} \) denote a finite set of probability estimators \(\{\bm{q}_n\}_{n=1}^N\). The goal is to identify, through theoretical and empirical analysis, the conditions under which selecting the estimator from \( \mathcal{Q} \) with the minimal noisy \(0\text{-}1\) risk leads to optimal clean \(0\text{-}1\) risk performance.

\section{Related Work}\label{ch7:sec:related_work}
A summary of the most relevant literature is given below. A comprehensive related work may be found in  Appendix~\ref{ch7:sec:further_related_work}.

\paragraph*{Over/Underfitting Persists:} The development of robust loss functions has been fruitful but challenges remain. While robust loss functions are generally less vulnerable to label noise than cross-entropy, they still exhibit overfitting or underfitting problems \citep{normalised_losses, GCE_Loss}. For instance, losses like Mean Absolute Error (MAE) and certain normalised losses tend to underfit, while others (e.g., backward corrections, Generalised Cross-Entropy (GCE), etc.) tend to overfit, albeit less severely than cross-entropy \cite{fprop}. Each loss function behaves differently depending on factors such as the number of classes, noise model, hyperparameters (e.g., learning rate and batch size), and datasets. Thus, a loss function may avoid overfitting in one setting but overfit or underfit in another. Consequently, it is crucial to develop methods to reliably measure model generalisation during training to assess whether overfitting (or underfitting) is occurring and to determine the optimal point for Early Stopping.

\paragraph*{Correction-Based Loss Functions}
An important category of robust loss functions is correction-based loss functions. Primarily applied in settings of class-conditional label noise, these methods infer the noise transition matrix ($\widehat{T}$) from the noisy data, using it to `correct' the loss function \citep{gold, anchor}.  Such approaches aim to account for and correct for the distorting influence of the label noise on the risk objective. This is motivated by the observation that, in the presence of label noise, the empirical risk is no longer a suitable proxy for the generalised risk. However, by applying a loss correction $L\mapsto L_{\text{corr.}}$ one ensures that
\begin{align*}
    \argmin_{q} \widehat{R}^{\eta}_{L_{\text{corr.}}}(q) &\approx \argmin_{q} {R}^{\eta}_{L_{\text{corr.}}}(q), \\&= \argmin_{q} {R}_L(q), 
\end{align*}
meaning that minimising the noisy risk objective aligns with generalisation to the clean (un-noised) distribution.
The \emph{forward} correction noises the outputs of the model by application of the estimated noise transition matrix $\widehat{T}$ \citep{fprop, anchor}. In contrast, the backward correction de-noised the noisy labels by applying $\widehat{T}^{-1}$ \citep{old_backward, fprop}. 

\paragraph*{ES for Label Noise} Limited work addresses Early Stopping (ES) for classification in the presence of label noise. \citet{EarlyStoppingNoisy} introduces Progressive Early Stopping (PES), a technique that trains initial neural network layers before implementing ES on later layers, which are more prone to overfitting, without using a validation set and halting training after a preset number of epochs. This approach raises questions about optimally determining duration without clean validation or test sets for hyperparameter tuning. PES involves multiple training stages and a final semi-supervised phase, which can be resource-intensive. \citet{robust_early_learning} presents `robust-early-learning,' which manages label noise by dividing parameters into critical and non-critical sets with different update rules and utilises ES on a noisy validation set. \citet{labelnoiseES} highlights ES's efficacy for symmetric label noise and theoretical effectiveness in one-hidden layer networks under specific data distribution assumptions. Our analysis differs notably in that their ES criteria are based on clustering properties of data distribution, not on generalisation from noisy data. In adversarial robustness, \citep{ES_AdvRobust} finds ES comparably effective to other robust methods when a cleanly labelled dataset is available for generalisation assessment. \citet{prestopping} introduces `Pre-stopping,’ using ES in a label noise robust pipeline, assuming a small cleanly labelled dataset for assessment. The backward correction allows ES without a clean validation set when the true transition matrix is known, as the empirical risk on a noisy set becomes an unbiased estimator of the un-noised risk, though its effectiveness depends on the accuracy of the transition matrix estimate and is limited in non-uniform label noise scenarios.

\paragraph*{Without Early Stopping} The majority of studies introducing label noise robust loss functions appear not to utilise any sort of Early Stopping at all \citep{GCE_Loss, sce, imae}. Instead, these approaches train for a pre-specified number of epochs, relying on the loss functions intrinsic robust properties to mitigate overfitting during the training run. However, multiple works including \citep{fprop, group_noise}, while not implementing Early Stopping, use a noisy validation set for hyperparameter tuning.

\paragraph*{Noisy and Clean $0\text{-}1$-Risk} A common Early-Stopping approach consists of monitoring the loss (typically cross-entropy) on the validation set \citep{early_stop_without_val}. In this study, however, we adopt the approach of monitoring the validation \emph{accuracy} ($0\text{-}1$-loss) \citep{robust_early_learning}. Consequently, understanding the relationship between the noisy and clean $0\text{-}1$ risks of a classifier is crucial. Prior research, such as \citep{robust_theory_earlier_binary, ghosh2015making}, has demonstrated the robustness of the $0\text{-}1$ loss to uniform symmetric noise within binary classification contexts. Building on this, \citep{robust_theory} establishes conditions under which the minimisers of both noisy and clean risks coincide for loss functions that exhibit a `symmetry' property. Given that the $0\text{-}1$ loss satisfies this symmetry criterion, the findings of \citet{robust_theory} apply to our work. While the primary focus of the referenced papers differs from ours, the results presented in Section~\ref{ch7:sec:theoretical} concerning symmetric label noise naturally extend the relationships they derived between noisy and clean risks.

\section{The Relationship Between Noisy and Clean $0\text{-}1$-Risk}\label{ch7:sec:theoretical}
This section provides conditions under which the minimiser of the noisy risk in a set of estimators $\mathcal{Q}$ coincides with the minimiser of the clean risk in $\mathcal{Q}$. The conditions are summarised into five `facts'. Precise statements and proofs are given in Appendix~\ref{ch7:sec:proofs}.

\subsection{Uniform Symmetric Label Noise}
By far the most studied label noise type is uniform symmetric label noise \citep{surveyDeep} which is label noise where the transition probabilities between every pair of distinct labels are the same. 

 \paragraph*{Fact 1:} Given a set of estimators $\mathcal{Q}\coloneq \{\bm{q}_n\}_{n=1}^N$ and any data-label distribution. The minimiser of the noisy risk within the set $\mathcal{Q}$ will also minimise the clean risk if the label noise is simultaneously uniform, symmetric and class-preserving.\footnote{Recall that we adopt the convention of using `risk' as shorthand for $0\text{-}1$ risk in this study.} Moreover, the noisy and clean risks are related by an affine linear relationship; 
     \begin{align}\label{eqn:sym_risk_relationship}
        R^{\eta}(\bm{q}) = R(\bm{q})\left(1-\frac{c\eta}{c-1}\right) +\eta.
    \end{align}

\begin{figure}[htbp]
\centering
  \centering
  \includegraphics[width=\linewidth]{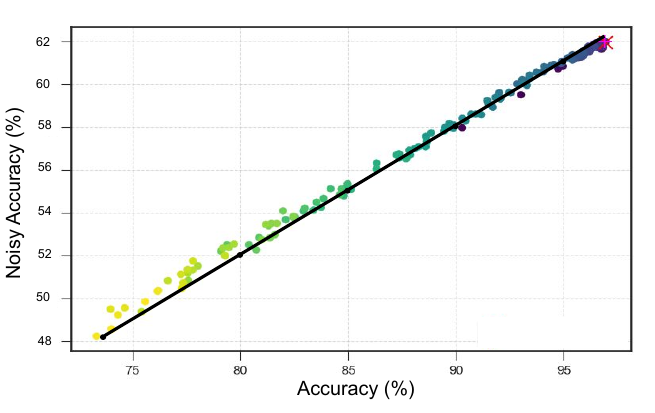}
\caption{Symmetric label noise - noisy vs clean accuracy: A classifier model is trained using cross-entropy loss on the MNIST dataset, corrupted by 36\% symmetric label noise. We plot the model's noisy and clean accuracies against each other at the end of each epoch, with early epochs coloured in dark blue and later epochs (around 100) in yellow. As expected, a linear relationship emerges between the noisy and clean accuracies. The theoretical relationship (Equation~\ref{eqn:sym_risk_relationship}) is depicted by the black line, showing near-perfect alignment between the experimental results and theoretical predictions.}
  \label{fig:clean_vs_noisy_line}
\end{figure}

Fact 1 asserts that, in the case of uniform symmetric label noise, the model that minimises the noisy risk within $\mathcal{Q}$ will also minimise the clean risk, and that the clean and noisy risks should be related by a linear mapping. Figure~\ref{fig:clean_vs_noisy_line} illustrates the relationship between noisy and clean accuracy for a neural network model at each training epoch on the symmetrically-noised MNIST dataset (36\% noise). The black line represents the expected theoretical relationship, demonstrating strong alignment with the experimental results. We may conclude from Fact 1 that under uniform symmetric label noise, a Noisy Early-Stopping policy is likely to be effective.

\subsection{Asymmetric and Non-Uniform Label Noise}
The theoretical result established by Fact 1 for uniform symmetric label noise is strong. Fact 2 establishes that uniform symmetric label noise is the \emph{only} noise model for which this strong result holds: For all other label noise models, there is \emph{no inherent reason to assume that performance evaluations on a noisy dataset will reliably reflect performance on the underlying clean distribution}.

\paragraph*{Fact 2:} Let $p(\widetilde{y}\mid y,x)$ be a label noise model with the property that, for any set of estimators $\mathcal{Q}\coloneq \{\bm{q}_n\}_{n=1}^N$ and any data-label distribution, the minimiser of the noisy risk within $\mathcal{Q}$ coincides with the minimiser of the clean risk within $\mathcal{Q}$; then $p(\widetilde{y}\mid y,x)$ \emph{must} describe uniform, symmetric and class-preserving label noise.

\paragraph{Example: Decision Tree Classifier} Figure~\ref{fig:both_classifiers} (Left) plots the noisy versus clean validation accuracy for \textbf{decision tree classifiers} of increasing depth trained on a dataset corrupted by pairwise (\emph{asymmetric}) label noise at $42\%$. Low-depth models are indicated with dark blue and deeper classifiers with yellow. We can see that, in sharp contrast to Figure~\ref{fig:clean_vs_noisy_line}, the relationship between clean and noisy validation accuracy is quite chaotic. In particular, the depth which optimises the noisy accuracy (indicated by the red vertical dotted line) is highly suboptimal for clean accuracy, achieving an accuracy 18\% lower than optimal. 

\begin{figure*}[ht] 
\centering

\includegraphics[height=6cm, width=0.48\textwidth]{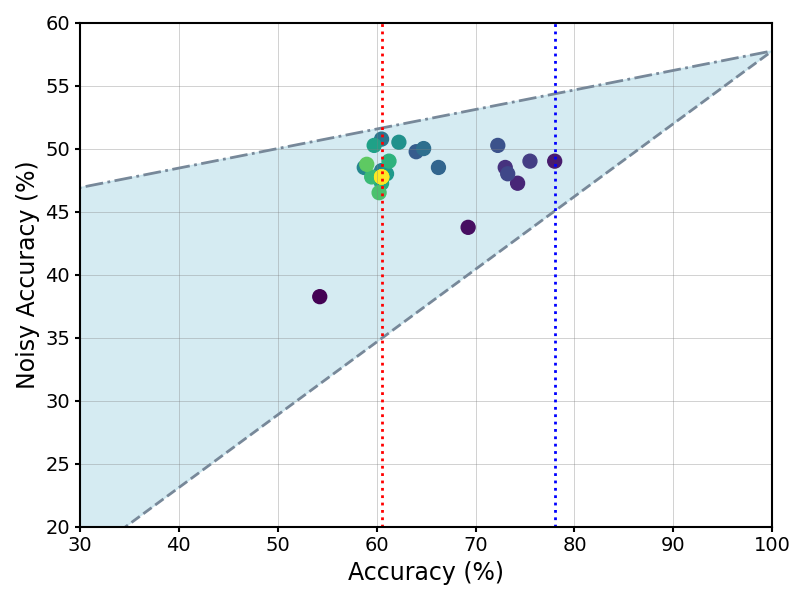} 
\hfill 
\includegraphics[height=6cm, width=0.48\textwidth]{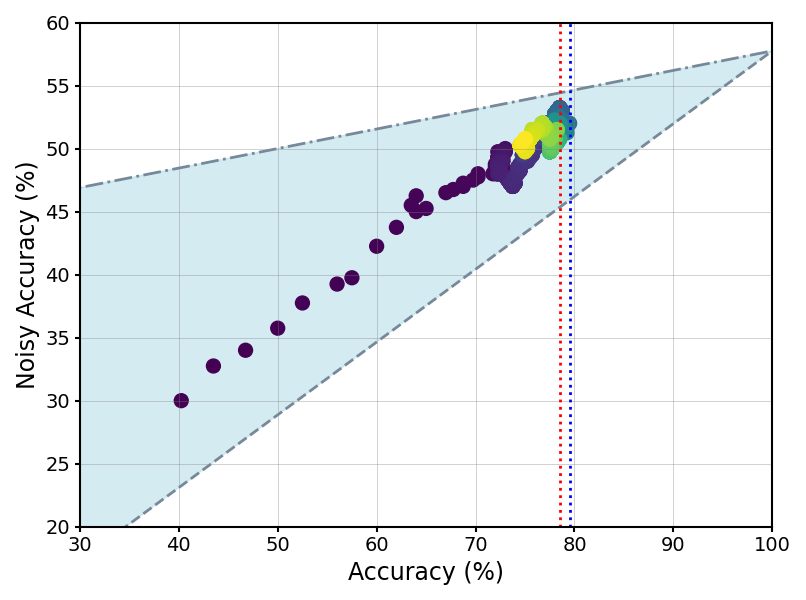} 

\caption{Noisy validation accuracy plotted against clean validation accuracy: \textbf{Left: \uline{Decision Tree Classifier}} at increasing depths, fitted to a dataset containing $42\%$ pairwise label noise. Shallow depth models are represented in blue, transitioning to yellow as depth increases. A red dotted vertical line highlights the classifier achieving the highest noisy validation accuracy, while a blue dotted vertical line marks the classifier with the highest clean validation accuracy. The significant horizontal gap between these lines illustrates the limited effectiveness of Noise Early Stopping (NES) for optimising the depth of decision trees under this type of label noise. \textbf{Right: \uline{Neural Network Classifier}} trained on the same noisy dataset. Early epochs are represented in blue, transitioning to yellow. A red dotted vertical line highlights the epoch with the highest noisy validation accuracy, while a blue dotted vertical line marks the epoch with the highest clean validation accuracy. The small horizontal gap between these lines illustrates the effectiveness of Noise Early Stopping (NES) for neural network models under similar noise conditions. For both graphs the light blue region represents bounds established by Fact 5, indicating that no model may achieve an accuracy/noisy-accuracy combination outside this region.}
\label{fig:both_classifiers}
\end{figure*}

\subsubsection{Additional Assumptions}
Under specific assumptions about the data distribution and the model set, optimising noisy accuracy can indeed align with optimising clean accuracy. Examples of such assumptions are demonstrated by the following facts:

\paragraph*{Fact 3:} 
Let $p(\widetilde{y}\mid y,x)$ be a class-preserving label noise model. If the set of estimators $\mathcal{Q}\coloneq \{\bm{q}_n\}_{n=1}^N$ contains a Bayes-optimal estimator (a minimiser of the clean risk over \emph{all} possible estimators) then this estimator will be a minimiser of the noisy risk in $\mathcal{Q}$. 

\paragraph*{Fact 4:} For \emph{non-uniform}, symmetric, class-preserving noise, the minimiser of the noisy risk within $\mathcal{Q}$ will coincide with the minimiser of the clean risk within $\mathcal{Q}$ if each of the estimators in $\mathcal{Q}$ are uncorrelated with the noise model. More generally, the worst-case performance of selecting the estimator with minimal noisy risk can be bounded in terms of the correlation between the noise model and the estimators. 

\paragraph*{Fact 5:} For uniform, \emph{asymmetric}, class-preserving label noise one can upper-bound the worst-case clean risk difference between the minimiser of the noisy risk ($\bm{q}^{\eta}_*$) and the minimiser of the clean risk ($\bm{q}_*$) in $\mathcal{Q}$. i.e. we upper bound $\vert R(\bm{q}^{\eta}_*)-R(\bm{q}_*)\vert$. The upper bound is given in terms of the maximum and minimum transition probabilities between classes and the minimal achievable noisy risk within $\mathcal{Q}$. 

\subsection{Expectations}
Based on Facts 1-5, we anticipate that Noisy Early Stopping (NES) should be effective for uniform symmetric label noise. However, for other noise models, the effectiveness of NES appears less promising. During typical gradient descent training of classifiers, achieving a Bayes-optimal classifier is unlikely, thus Fact 3 does not hold. Furthermore, Fact 4 applies to non-uniform symmetric (not \emph{asymmetric}) noise. Additionally, it's uncertain whether the classifiers in $\mathcal{Q}$ can remain uncorrelated with the noise model, given that they are trained on data affected by this noise. Moreover, Figure~\ref{fig:both_classifiers} (Left) illustrates, that NES performs poorly for decision tree classifiers for asymmetric label noise. Coupled with the poor bounds outlined in Fact 5, (discussed in Appendix~\ref{ch7:sec:bounds_discussion}), it is reasonable to expect that NES for neural network classifier will significantly underperform compared to ES under most label noise conditions. 

\section{Experiments}\label{ch7:sec:experiments}

\begin{figure*}[btp]
\centering
  \begin{minipage}[b]{0.49\textwidth} 
    \centering
    \includegraphics[width=\textwidth, keepaspectratio]{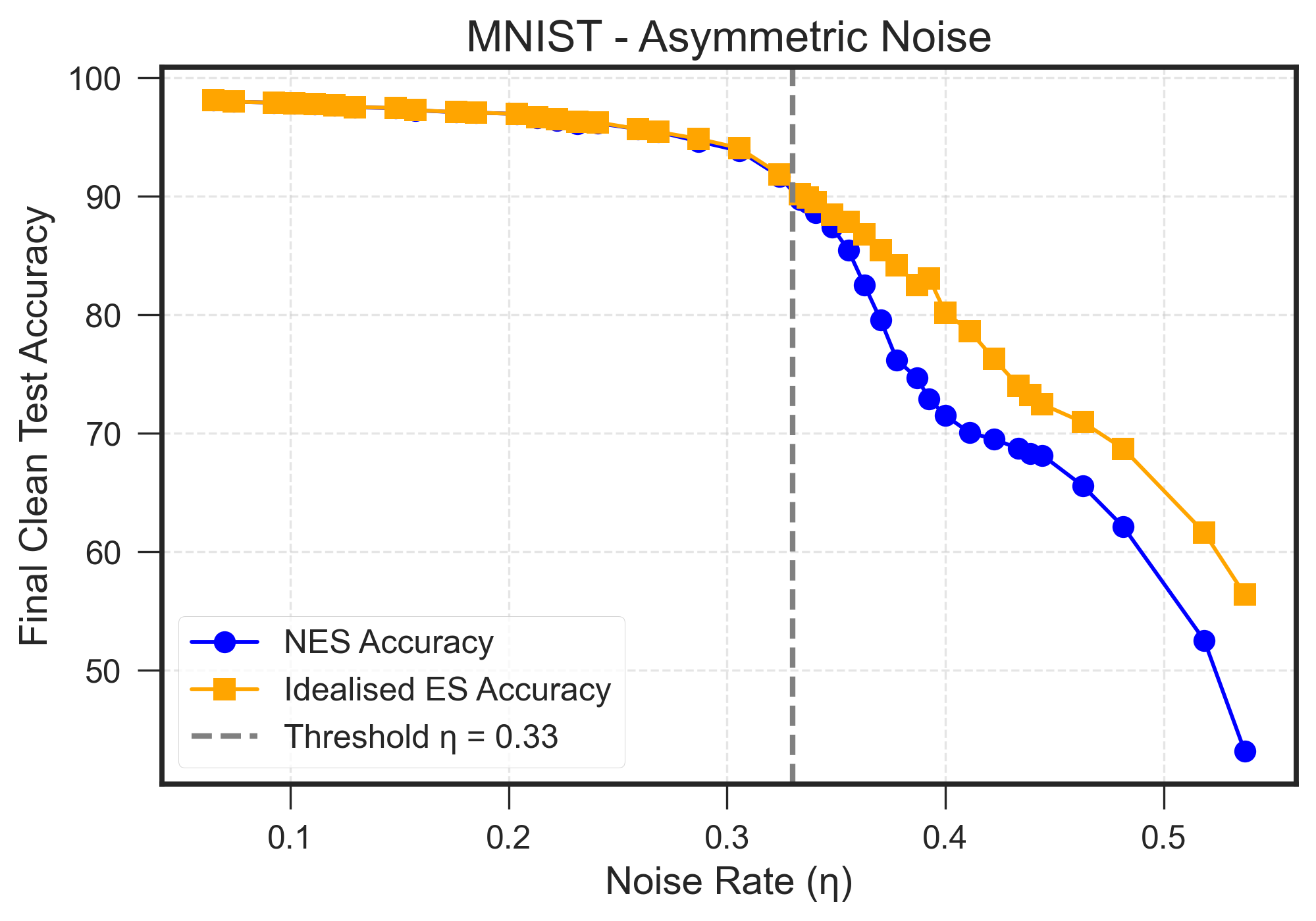}
  \end{minipage}
  \hfill 
  \begin{minipage}[b]{0.49\textwidth}  
    \centering
    \includegraphics[width=\textwidth, keepaspectratio]{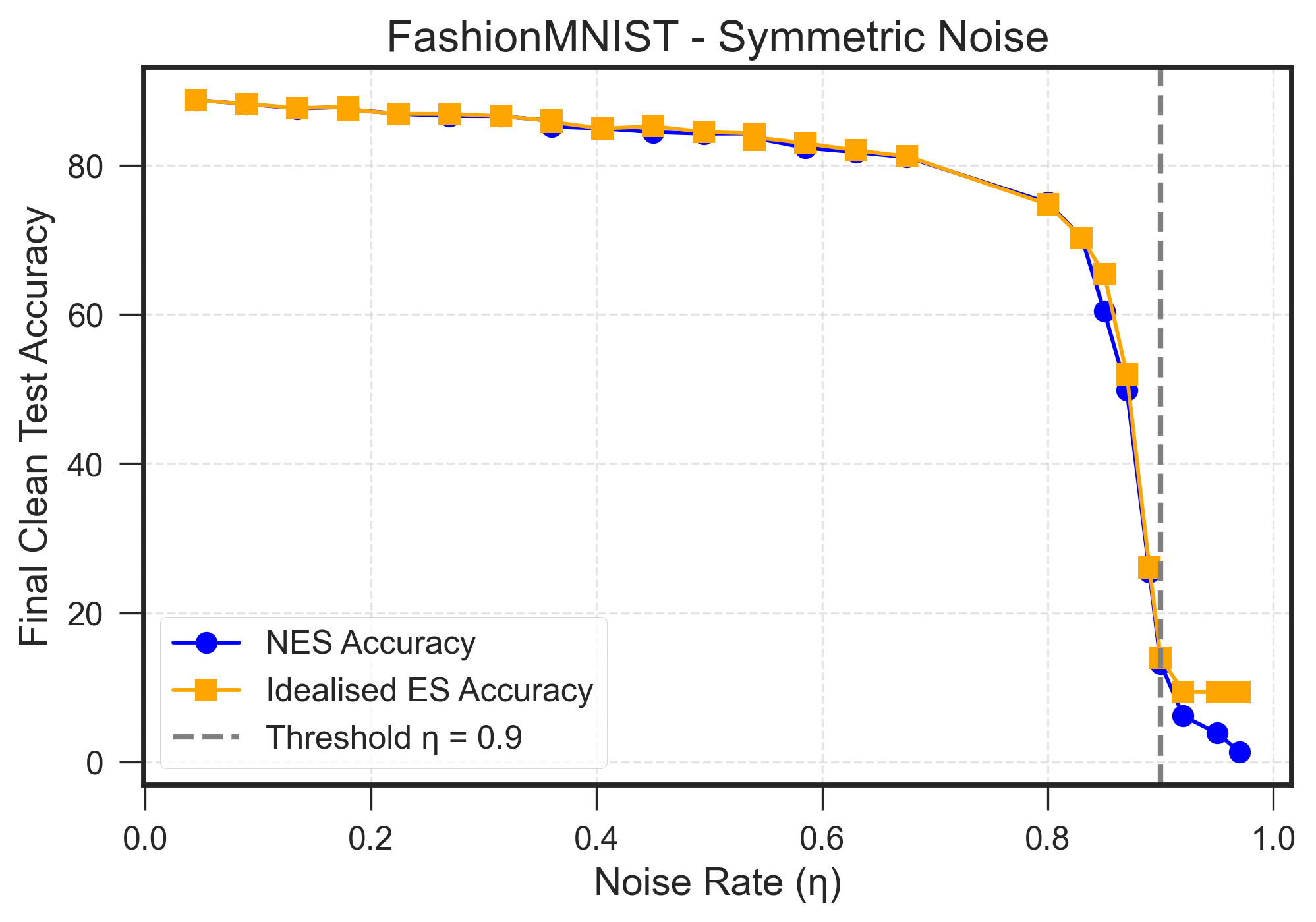} 
  \end{minipage}
\caption{Comparison of NES (blue) and clean Early Stopping (ES) (yellow) across increasing noise rates ($\eta$). The plots show final clean test accuracy for models trained on the asymmetrically-noised MNIST dataset (left) and symmetrically-noised Fashion-MNIST dataset (right) when implementing each Early Stopping policy. Vertical lines at $\eta = 0.33$ for MNIST and $\eta = 0.9$ for Fashion-MNIST mark the thresholds where noise ceases to be class-preserving. Below these thresholds, NES closely aligns with ES, demonstrating its robustness under varying levels of label noise.}
\label{ch7:fig:Fashion_and_MNIST_class_pres}
\end{figure*}

\subsection{Experiment Details} We evaluate Noisy Early Stopping (NES) against clean Early Stopping (ES), which uses a cleanly-labelled validation set, and Without-Early-Stopping (WES), the standard for noisy label training. 

\subsubsection{Experiment Setup} We prepare our experiments using cleanly-labelled datasets divided into training, validation, and test segments. For Noisy Early Stopping (NES), we inject synthetic label noise into both the training and validation sets. We then train neural network classifiers on this noisy training data for 100-150 epochs, evaluating them at each epoch using the noisy validation set accuracy. If there’s no improvement in validation accuracy for ten epochs, we halt training and revert to the best model observed (patience parameter $P=10$).

For clean Early Stopping (ES), we follow a similar training procedure, but the validation set remains clean and unaffected by synthetic noise. In contrast, the Without-Early-Stopping (WES) method involves continuous training through the predetermined epoch count without monitoring validation performance. All methods are ultimately evaluated on the clean test set.

\paragraph*{Remark:} In real-world scenarios, a clean validation set may not be available. The comparison with ES serves to contextualise NES's performance relative to an idealised `gold-standard' baseline \citep{early_stop_without_val}. 

\subsubsection{Datasets and Noise Models} We evaluate the performance of NES/ES/WES across 10 noisy datasets: FashionMNIST, MNIST, CIFAR10, and CIFAR100, all corrupted by symmetric label noise, as well as a variety of non-symmetric label noise types. These include non-uniform asymmetrically noised EMNIST (NU-EMNIST), non-uniform asymmetric MNIST (NU-MNIST), and uniform-asymmetric noise variants of FashionMNIST, MNIST, CIFAR10, and CIFAR100. Full details of these label noise models are provided in Appendix~\ref{app:dataset_models}. Although Clothing1M and Animals-10 are common datasets for evaluating label noise-robust methods, they lack cleanly labelled validation sets and therefore are not suitable for testing the core hypothesis of this work.

\paragraph*{Loss Functions} We evaluate the performance of NES, ES, and WES using six different loss functions: Cross-Entropy (\textbf{CE}), and five popular robust loss functions: \textbf{MSE} \citep{janocha2017loss}, \textbf{GCE} \citep{GCE_Loss}, \textbf{SCE} \citep{sce}, forward-corrected CE (\textbf{FCE}) \citep{fprop}, and backward-corrected CE (\textbf{BCE}) \citep{natarajan2013learning}. The transition matrix used in the backward and forward-corrected loss functions is set to the true noise model used to construct the label noise, except for non-uniform label noise, where it is set to uniform symmetric at the given noise rate.

\subsection{NES Results} Table~\ref{tab:performance-metrics} presents the results of each method (NES/WES/ES) across the datasets for each of the six loss functions. The table shows the average final clean test accuracy and the standard deviation over three runs. The results of our proposed NES method are shaded in light blue. Where the performance of NES exceeds that of ES, or their uncertainty intervals overlap, the corresponding entry is \textbf{bolded}. This occurs in $\bm{93\%}$ of settings, indicating that NES performs nearly as well as ES despite using a noisy validation set instead of a clean one. These results are consistent across different loss functions and hold for datasets corrupted by uniform symmetric label noise as well as those affected by other noise types, despite the absence of strong theoretical guarantees.

Table~\ref{tab:performance-metrics} also presents the results of the standard Without Early Stopping (WES) approach for these datasets. The final clean test accuracies for WES are often substantially lower than those for ES and NES, highlighting the importance of early-stopping approaches to mitigate overfitting, even when using robust loss functions. NES outperforms WES in $\bm{75\%}$ of the experimental settings.

\subsection{Plots and Figures} In Section~\ref{ch7:sec:theoretical}, we evaluated the effectiveness of Noisy Early Stopping (NES) for determining the optimal depth of a decision tree classifier trained with noisy labels. NES proved ineffective for asymmetric label noise. Figure~\ref{fig:both_classifiers} (Right) explores the application of NES to a neural network classifiers for this same noisy dataset. Contrary to the findings with decision trees, we observe that for the neural network, the noisy and clean validation accuracies peak almost simultaneously, indicating the effectiveness of NES in this context.

Additional plots that further illustrate the relationship between clean and noisy accuracy across various datasets and noise models, and demonstrate the effectiveness of NES for neural network models, are presented in Appendix~\ref{ch7:sec:additional}.

\subsubsection{The Importance of Being Class-Preserving}
In Figure~\ref{ch7:fig:Fashion_and_MNIST_class_pres} we visualise the performance of NES as a function of the noise rate. We plot the final (clean) test accuracy of models trained using NES in blue and the final (clean) test accuracy of ES in yellow against increasing noise rates for two noisy datasets (AsymMNIST left and FashionMNIST right). Beyond a certain threshold noise level, the class-preserving condition is violated, highlighted by a vertical dashed line. Until this point, noisy and clean ES perform identically, with an immediate divergence once the class-preserving property no longer holds. This finding reaffirms the necessity for label noise to be class-preserving for NES to be an effective substitute for clean ES. When the noise rate is increased to the point where the label noise is no longer class-preserving the original signal relating the data points to the label is essentially destroyed. As a result the noisy accuracy of a model no longer aligns with its clean accuracy.

\subsection{Implications of Findings}\label{ch7:sec:discussion}
The results shown in the Table~\ref{tab:performance-metrics} and in Figure~\ref{fig:both_classifiers} (right) illustrate that NES performs well and is about as effective as the clean Early Stopping idealised baseline. While for uniform symmetric label noise this was to be expected, the fact that NES works for other noise types is remarkable. This means that even when no clean validation set is available we can early stop and obtain near-optimal clean test performance. The disparity in performance between the Without Early-Stopping baseline and NES underscores the importance of employing Early Stopping when data is corrupted by label noise. Therefore the discovery that we can do almost optimal Early Stopping without the need to construct a cleanly labelled validation set is a meaningful finding. 

\paragraph*{Example Application:} Consider a scenario with a large dataset that requires labelling. We assemble a diverse group of human labellers, from novices to experts, to annotate the dataset. This yields a large dataset with noisy labels where we do not know the precise noise model or the noise rate. We assume that the label noise obeys the class-preserving condition. We select a standard loss function and we wish to train a neural network model on the noisy dataset. However, we are concerned about overfitting. We randomly split off a portion of the noisy dataset to form a noisy validation set. During training, we monitor the accuracy on this noisy validation set and cease training once noisy validation begins to decline (subject to some patience parameter). Our experimental results suggest that this would result in almost perfect Early Stopping, achieving near-peak clean test accuracy and avoiding overfitting. 

\begin{table*}[!ht]
\centering
\caption{Performance Metrics across Different Datasets and Methods. This table presents a comparison of three methods—NES (Noisy Early Stopping), ES (Early Stopping), and WES (Without Early Stopping)—using six commonly used loss functions across various noisy datasets. NES, highlighted in blue, is our proposed method, placed first to highlight its performance. Bold values indicate where NES overlaps or exceeds the performance of ES (93\% of cases), an idealised method that uses a clean validation set, providing a context for NES's efficacy when such a set is unavailable. WES, typically used in training with robust loss functions without early stopping, is outperformed by NES in 75\% of cases, demonstrating the significant benefits of early-stopping strategies.}
\label{tab:performance-metrics}
\begin{tabular}{@{}lccccccc@{}} 
\toprule
\multirow{2}{*}{\textbf{Dataset}} &  & \multicolumn{6}{c}{\textbf{Loss Function}} \\
\cmidrule(lr){3-8}
& & \textbf{CE} & \textbf{MSE} & \textbf{GCE} & \textbf{SCE} & \textbf{FCE} & \textbf{BCE} \\
\midrule
\multirow{3}{*}{\textbf{MNIST}} 
& \cellcolor{lightblue} NES & \cellcolor{lightblue} $\bm{88.43}_{\pm 0.65}$ & \cellcolor{lightblue} $\bm{88.78}_{\pm 0.57}$ & \cellcolor{lightblue} $\bm{92.37}_{\pm 0.72}$ & \cellcolor{lightblue} $\bm{89.91}_{\pm 0.28}$ & \cellcolor{lightblue} $\bm{92.45}_{\pm 0.58}$ & \cellcolor{lightblue} $\bm{88.50}_{\pm 1.06}$ \\
& WES & $34.18_{\pm 3.10}$ & $52.35_{\pm 1.74}$ & $66.64_{\pm 1.11}$ & $32.20_{\pm 1.37}$ & $86.83_{\pm 0.83}$ & $77.23_{\pm 1.41}$ \\
& \textcolor{gray}{ES}  & \textcolor{gray}{$88.34_{\pm 0.74}$} & \textcolor{gray}{$88.78_{\pm 0.57}$} & \textcolor{gray}{$92.63_{\pm 0.38}$} & \textcolor{gray}{$89.91_{\pm 0.28}$} & \textcolor{gray}{$92.50_{\pm 0.50}$} & \textcolor{gray}{$89.94_{\pm 0.31}$} \\
\midrule
\multirow{3}{*}{\textbf{Fashion}} 
& \cellcolor{lightblue} NES & \cellcolor{lightblue} $\bm{83.21}_{\pm 0.37}$ & \cellcolor{lightblue} $83.48_{\pm 0.07}$ & \cellcolor{lightblue} $\bm{84.77}_{\pm 0.28}$ & \cellcolor{lightblue} $\bm{84.17}_{\pm 0.31}$ & \cellcolor{lightblue} $\bm{84.43}_{\pm 0.26}$ & \cellcolor{lightblue} $\bm{80.35}_{\pm 0.54}$ \\
& WES & $49.94_{\pm 0.28}$ & $68.17_{\pm 1.60}$ & $80.67_{\pm 0.46}$ & $51.06_{\pm 0.42}$ & $78.74_{\pm 0.57}$ & $68.58_{\pm 0.30}$ \\
& \textcolor{gray}{ES}  & \textcolor{gray}{$82.96_{\pm 0.19}$} & \textcolor{gray}{$83.69_{\pm 0.09}$} & \textcolor{gray}{$84.77_{\pm 0.28}$} & \textcolor{gray}{$83.48_{\pm 0.90}$} & \textcolor{gray}{$84.47_{\pm 0.32}$} & \textcolor{gray}{$79.84_{\pm 0.16}$} \\
\midrule
\multirow{3}{*}{\textbf{CIFAR10}} 
& \cellcolor{lightblue} NES & \cellcolor{lightblue} $59.46_{\pm 2.51}$ & \cellcolor{lightblue} $\bm{60.63}_{\pm 0.33}$ & \cellcolor{lightblue} $\bm{72.67}_{\pm 2.04}$ & \cellcolor{lightblue} $\bm{68.82}_{\pm 1.39}$ & \cellcolor{lightblue} $\bm{70.63}_{\pm 1.63}$ & \cellcolor{lightblue} $\bm{59.47}_{\pm 4.19}$ \\
& WES & $37.23_{\pm 1.43}$ & $34.54_{\pm 0.88}$ & $64.21_{\pm 2.32}$ & $38.69_{\pm 0.39}$ & $54.43_{\pm 0.61}$ & $42.32_{\pm 1.37}$ \\
& \textcolor{gray}{ES}  & \textcolor{gray}{$63.70_{\pm 0.98}$} & \textcolor{gray}{$61.22_{\pm 0.88}$} & \textcolor{gray}{$72.65_{\pm 2.05}$} & \textcolor{gray}{$67.73_{\pm 1.19}$} & \textcolor{gray}{$70.74_{\pm 1.78}$} & \textcolor{gray}{$59.47_{\pm 4.19}$} \\
\midrule
\multirow{3}{*}{\textbf{CIFAR100}} 
& \cellcolor{lightblue} NES & \cellcolor{lightblue} $\bm{51.86}_{\pm 1.71}$ & \cellcolor{lightblue} $\bm{41.80}_{\pm 0.99}$ & \cellcolor{lightblue} $\bm{63.74}_{\pm 1.55}$ & \cellcolor{lightblue} $\bm{53.82}_{\pm 2.31}$ & \cellcolor{lightblue} $\bm{65.98}_{\pm 0.45}$ & \cellcolor{lightblue} $\bm{55.58}_{\pm 0.25}$ \\
& WES & $38.96_{\pm 1.21}$ & $48.14_{\pm 0.85}$ & $66.53_{\pm 0.81}$ & $39.59_{\pm 0.47}$ & $47.33_{\pm 1.12}$ & $44.82_{\pm 0.56}$ \\
& \textcolor{gray}{ES}  & \textcolor{gray}{$51.31_{\pm 2.69}$} & \textcolor{gray}{$43.33_{\pm 0.98}$} & \textcolor{gray}{$62.91_{\pm 1.87}$} & \textcolor{gray}{$54.07_{\pm 2.14}$} & \textcolor{gray}{$66.34_{\pm 0.66}$} & \textcolor{gray}{$55.58_{\pm 0.25}$} \\
\midrule
\multirow{3}{*}{\textbf{AsymMNIST}} 
& \cellcolor{lightblue} NES & \cellcolor{lightblue} $\bm{95.27}_{\pm 0.93}$ & \cellcolor{lightblue} $\bm{94.38}_{\pm 1.25}$ & \cellcolor{lightblue} $\bm{97.53}_{\pm 0.18}$ & \cellcolor{lightblue} $\bm{96.87}_{\pm 0.42}$ & \cellcolor{lightblue} $\bm{97.98}_{\pm 0.12}$ & \cellcolor{lightblue} $\bm{94.69}_{\pm 1.35}$ \\
& WES & $83.32_{\pm 3.81}$ & $84.03_{\pm 2.38}$ & $94.83_{\pm 0.18}$ & $78.17_{\pm 0.30}$ & $97.53_{\pm 0.17}$ & $83.56_{\pm 1.22}$ \\
& \textcolor{gray}{ES}  & \textcolor{gray}{$95.28_{\pm 0.94}$} & \textcolor{gray}{$95.12_{\pm 0.69}$} & \textcolor{gray}{$97.53_{\pm 0.10}$} & \textcolor{gray}{$96.87_{\pm 0.42}$} & \textcolor{gray}{$98.20_{\pm 0.15}$} & \textcolor{gray}{$95.70_{\pm 0.10}$} \\
\midrule
\multirow{3}{*}{\textbf{AsymFashion}} 
& \cellcolor{lightblue} NES & \cellcolor{lightblue} $\bm{72.76}_{\pm 2.59}$ & \cellcolor{lightblue} $\bm{73.17}_{\pm 1.99}$ & \cellcolor{lightblue} $\bm{74.56}_{\pm 0.82}$ & \cellcolor{lightblue} $\bm{73.68}_{\pm 2.99}$ & \cellcolor{lightblue} $\bm{77.09}_{\pm 1.07}$ & \cellcolor{lightblue} $\bm{77.38}_{\pm 1.95}$ \\
& WES & $68.61_{\pm 1.38}$ & $68.57_{\pm 0.70}$ & $70.74_{\pm 0.22}$ & $68.11_{\pm 0.87}$ & $77.65_{\pm 1.05}$ & $71.35_{\pm 0.31}$ \\
& \textcolor{gray}{ES}  & \textcolor{gray}{$74.24_{\pm 0.97}$} & \textcolor{gray}{$73.46_{\pm 1.27}$} & \textcolor{gray}{$73.50_{\pm 1.17}$} & \textcolor{gray}{$75.03_{\pm 0.65}$} & \textcolor{gray}{$77.92_{\pm 0.58}$} & \textcolor{gray}{$71.95_{\pm 8.99}$} \\
\midrule
\multirow{3}{*}{\textbf{AsymCIFAR10}}
& \cellcolor{lightblue} NES & \cellcolor{lightblue} $\bm{83.59}_{\pm 0.51}$ & \cellcolor{lightblue} $\bm{82.62}_{\pm 1.52}$ &\cellcolor{lightblue}  $\bm{82.50}_{\pm 1.16}$ & \cellcolor{lightblue} $\bm{81.43}_{\pm 0.67}$ & \cellcolor{lightblue} $\bm{84.56}_{\pm 1.31}$ & \cellcolor{lightblue} $\bm{74.92}_{\pm 3.64}$ \\
& WES & $84.49_{\pm 2.02}$ & $82.80_{\pm 1.21}$ & $86.34_{\pm 0.57}$ & $85.60_{\pm 0.64}$ & $89.62_{\pm 0.49}$ & $84.78_{\pm 0.15}$ \\
& \textcolor{gray}{ES}  & \textcolor{gray}{$83.59_{\pm 0.51}$} & \textcolor{gray}{$82.62_{\pm 1.52}$} & \textcolor{gray}{$82.07_{\pm 1.13}$} & \textcolor{gray}{$81.43_{\pm 0.67}$} & \textcolor{gray}{$85.35_{\pm 0.89}$} & \textcolor{gray}{$76.90_{\pm 5.03}$} \\
\midrule
\multirow{3}{*}{\textbf{AsymCIFAR100}} 
& \cellcolor{lightblue} NES & \cellcolor{lightblue} $\bm{73.28}_{\pm 0.69}$ & \cellcolor{lightblue} $\bm{54.56}_{\pm 0.47}$ & \cellcolor{lightblue} $\bm{57.97}_{\pm 4.82}$ & \cellcolor{lightblue} $\bm{72.04}_{\pm 2.01}$ & \cellcolor{lightblue} $\bm{66.71}_{\pm 2.21}$ & \cellcolor{lightblue} $\bm{41.45}_{\pm 1.09}$ \\
& WES & $70.58_{\pm 0.73}$ & $57.25_{\pm 0.75}$ & $64.77_{\pm 0.61}$ & $68.58_{\pm 0.30}$ & $64.49_{\pm 1.18}$ & $48.86_{\pm 0.83}$ \\
& \textcolor{gray}{ES}  & \textcolor{gray}{$73.28_{\pm 0.69}$} & \textcolor{gray}{$54.94_{\pm 0.84}$} & \textcolor{gray}{$62.69_{\pm 2.91}$} & \textcolor{gray}{$72.04_{\pm 2.01}$} & \textcolor{gray}{$67.41_{\pm 1.37}$} & \textcolor{gray}{$42.28_{\pm 0.84}$} \\
\midrule
\multirow{3}{*}{\textbf{NU-MNIST}} 
& \cellcolor{lightblue} NES & \cellcolor{lightblue} $\bm{97.41}_{\pm 0.29}$ & \cellcolor{lightblue} $\bm{97.68}_{\pm 0.31}$ & \cellcolor{lightblue} $\bm{98.10}_{\pm 0.11}$ & \cellcolor{lightblue} $\bm{98.03}_{\pm 0.18}$ & \cellcolor{lightblue} $\bm{98.01}_{\pm 0.06}$ & \cellcolor{lightblue} $\bm{97.07}_{\pm 0.08}$ \\
& WES & $92.79_{\pm 0.25}$ & $96.16_{\pm 0.07}$ & $98.25_{\pm 0.03}$ & $92.99_{\pm 0.73}$ & $97.70_{\pm 0.13}$ & $95.62_{\pm 0.66}$ \\
& \textcolor{gray}{ES}  & \textcolor{gray}{$97.45_{\pm 0.28}$} & \textcolor{gray}{$97.69_{\pm 0.32}$} & \textcolor{gray}{$98.10_{\pm 0.11}$} & \textcolor{gray}{$98.20_{\pm 0.20}$} & \textcolor{gray}{$98.01_{\pm 0.06}$} & \textcolor{gray}{$97.07_{\pm 0.08}$} \\
\midrule
\multirow{3}{*}{\textbf{NU-EMNIST}} 
& \cellcolor{lightblue} NES & \cellcolor{lightblue} $\bm{91.55}_{\pm 2.55}$ & \cellcolor{lightblue} $53.93_{\pm 27.93}$ & \cellcolor{lightblue} $11.93_{\pm 0.16}$ & \cellcolor{lightblue} $\bm{88.51}_{\pm 0.64}$ & \cellcolor{lightblue} $\bm{68.25}_{\pm 4.17}$ & \cellcolor{lightblue} $\bm{89.94}_{\pm 0.94}$ \\
& WES & $91.50_{\pm 2.70}$ & $76.33_{\pm 1.16}$ & $15.14_{\pm 1.58}$ & $86.60_{\pm 0.57}$ & $71.02_{\pm 3.76}$ & $89.22_{\pm 0.46}$ \\
& \textcolor{gray}{ES}  & \textcolor{gray}{$92.35_{\pm 2.28}$} & \textcolor{gray}{$75.10_{\pm 0.91}$} & \textcolor{gray}{$14.38_{\pm 2.13}$} & \textcolor{gray}{$89.02_{\pm 0.38}$} & \textcolor{gray}{$70.88_{\pm 3.69}$} & \textcolor{gray}{$90.81_{\pm 0.20}$} \\
\addlinespace
\midrule
\bottomrule
\end{tabular}
\end{table*}

\section{Conclusion}\label{ch7:sec:conclusions}

\subsection{Summary}
The purpose of this study has been to look at the relationship between the noisy and clean risk of a classifier during training when a dataset is corrupted by label noise. Our primary goal has been to investigate, empirically and theoretically, whether noisy accuracy can be used as an effective criterion for Early Stopping. 

\paragraph*{What We've Shown} In Section~\ref{ch7:sec:theoretical} we gave some theoretical insights regarding NES. We showed that NES should be effective when label noise is uniform and symmetric. We proved that for all other noise types, these strong guarantees enjoyed by uniform symmetric noise do not apply. Consequently, our theoretical results implied that NES might not be effective unless label noise was uniform and symmetric. Unexpectedly, Section~\ref{ch7:sec:experiments} demonstrated empirically that using noisy accuracy as an Early-Stopping criterion is highly effective, performing comparably with clean Early Stopping across datasets and noise types. This finding is both intriguing and practical, as it suggests that practitioners can apply near-optimal Early Stopping strategies even in the absence of clean data. However, the surprising nature of these results indicates that our theoretical analysis in Section~\ref{ch7:sec:theoretical} is incomplete, necessitating further research to fully understand why NES is so effective. A potential avenue for future exploration into the efficacy of NES is discussed in Appendix~\ref{ch7:sec:why_nes_works}. 

\subsection{Limitations and Future Directions}
While our finding that NES is effective is useful and interesting there are several limitations which we enumerate in this section along with possible research directions our study opens up. 

\subsubsection{Future Work}
\paragraph*{The Class-Preserving Assumption} Throughout this work we assume that the label noise we are dealing with is class-preserving. As we can see in Figure~\ref{ch7:fig:Fashion_and_MNIST_class_pres}, when this condition is violated NES and ES begin to diverge. As we have tried to emphasise, this noise condition is broad, and the vast majority of label noise studied in the literature is of this variety. Nevertheless, rare instances of non-class-preserving label noise have been studied in prior work \citep{fprop}. Can Noisy Early-Stopping be extended to these settings? An appeal of NES is that it can be used even when the exact noise model is not known. It is this agnosticism which mandates that we make some limiting assumption about the set of admissible noise models (see no-free-lunch theorems \citep{wolpert1997noFreeLunch}). Consequently, to extend NES to the non-class-preserving setting it becomes necessary to have some knowledge about the structure of the noise model. Although the non-class-preserving setting is outside the scope of this work we speculate that if one has a Fisher consistent loss function (see section~\ref{ch7:sec:background} e.g. CE) and one corrects the loss $(L\mapsto L_F)$ in such a way that a global minimiser of the noisy $L_F-$risk is Bayes-optimal for the clean distribution then NES will be effective when training using $L_F$. Nevertheless, this is speculation and will need to be looked at in further work.  

\paragraph*{Closed-Set} We reiterate also that these findings apply to closed-set label noise - label noise where the true label of each sample lies within the given label set. (For example an image of Tiger Woods in an animal dataset and labelled as `tiger' is open-set label noise not closed-set). Open-set label noise is outside the scope of this work and should be investigated in further studies. 

\paragraph*{Revaluation of Robust Loss Functions} Finally, we feel that it would be useful to make a comprehensive comparison of existing robust loss functions across the standard noisy dataset benchmarks now that we have an effective method of Early Stopping that works across loss functions. 

\subsubsection{Limitations}
\paragraph*{Data Scarcity} In a setting in which the noisy data is scarce it may not be practical to separate some of this data into a validation set as this could leave an impractically small training set. Our results apply in settings where there is ample noisy data to form a validation set while retaining a sufficiently large training set. 

\paragraph*{Double Descent} Previous research has shown that when labels contain noise an epochwise double descent phenomenon can manifest \citep{epochwise, nakkiran2021deep_epochwise_orig}. While the clean test error rises after a certain number of epochs, it proceeds to dip again if training continues beyond this point. Simultaneously, the clean test accuracy undergoes a second increase. When the noise rate is low $\approx 5\%$ this second increase in test accuracy can be larger than the first peak. A consequence of this is that Early Stopping with a small patience parameter would stop training before this second larger peak is obtained. This is a problem of Early Stopping more broadly rather than Noisy Early Stopping specifically, nevertheless, it is important to be cognisant of this double descent phenomenon in the low noise rate setting and to select a suitably large patience parameter. 

\paragraph*{Early Stopping} In our theory section (Section~\ref{ch7:sec:theoretical}) we give conditions under which the minimiser of the noisy risk will minimise the clean risk, arguing that under these conditions Noisy Early Stopping will be effective. This relies on the assumption that we can precisely estimate the noisy risk at each epoch of training. In practice, we estimate the noisy risk by computing the misclassification rate on a held-out noisy validation set which remains constant through training. Since we are using a finite test set there will be some uncertainty in this noisy risk approximation. Moreover, since we use the \emph{same} validation set during training this introduces covariances between noisy risk estimations for each epoch. If the validation set is sufficiently large then these variances/covariances are small. However, when the validation set is insufficiently large both the variances and covariances will be non-trivial. In this regime selecting the model with the lowest estimated noisy risk may not correspond to selecting the model with the lowest noisy risk within $\mathcal{Q}$.

\paragraph*{Peak Test Accuracy During Training} Noisy Early Stopping allow us to cease training at or near the point where clean test accuracy peaks. The effectiveness of NES therefore depends on the height of this peak; in particular, if the classifier being learned does not generalise well at any epoch during training then NES cannot alter this fact. Thus, while NES allows us to get the best performance out of a particular method it is limited by the performance of the given method. As established in the relevant literature, different loss functions attain different peak test accuracies during training in the presence of noisy labels \citep{L2, GCE_Loss}. Since our work provides a way to obtain this peak model, it reinforces the importance of continuing to develop and improve robust loss functions. We believe that the findings of this study should redirect future efforts from developing loss functions that prevent overfitting to those that enhance the peak performance attained during training.

\bibliographystyle{apalike}
\bibliography{earlyStopping/earlyStopping}


\newpage
\appendix
\onecolumn

\section{Notation and Terminology}
\begin{longtable}{>{\raggedright\arraybackslash}p{2cm} p{12cm}}
\caption{Notation Table: Table summarising the notation used in this study.}\label{ch1:tab:notation} \\
\hline
\textbf{Symbol} & \textbf{Description} \\
\hline
\endfirsthead

\hline
\textbf{Symbol} & \textbf{Description} \\
\hline
\endhead

\hline
\endfoot

\hline
\endlastfoot

$c$ & Number of classes/labels. \\
$\mathcal{X}$ & The data domain, a subset of $\mathbb{R}^d$. \\
$\mathcal{Y}$ & The label space, defined as ${1, 2, 3, \ldots, c}$. \\
$\Delta$ & Probability simplex: The set of vectors $(p_1, p_2, \ldots, p_c)$ where each $p_i \geq 0$ and $\sum p_i = 1$. \\
$\bm{q}$ & A probability vector representing a forecast. \\
$\bm{p}$ & A probability vector representing ground-truth probabilities. \\
$\bm{q}:\mathcal{X}\rightarrow \Delta$ & A probability estimator model producing a forecast at each point in $\mathcal{X}$.\\
$\bm{p}(y\mid x)$ & The vector representing the class posterior probabilities at $x$, expressed as $\bm{p}(y\mid x) = (p(y=1\mid x), p(y=2\mid x), \ldots, p(y=c\mid x))$. \\
$f$ & A classifier function mapping each point in $\mathcal{X}$ to a label in $\mathcal{Y}$. \\
$L$ & The loss function used to evaluate the accuracy of predictions against actual labels. \\
$\bm{L}$ & The vector-valued function of the loss function $L$, where $\bm{L}(\bm{q}) = (L(\bm{q},1), \ldots, L(\bm{q},c))$. \\
$R_{L}(\bm{q})$ & The $L$-risk of an estimator $\bm{q}$. \\
$R_{L}(\bm{q})(x)$ & The \emph{pointwise} $L$-risk of an estimator $\bm{q}$ at $x$. \\
$R^{\eta}_{L}(\bm{q})$ & The noisy $L$-risk of an estimator $\bm{q}$. \\
$\mathcal{H}$ or $\mathcal{H}_{L}$ & The entropy function corresponding to the loss function $L$. \\
$H_{L}(\bm{p},\bm{q})$ & The expected $L$-loss for a forecast $\bm{q}$ given the true label distribution $\bm{p}$. \\
$\eta$ & The noise rate of the label noise model. \\
$y, \widetilde{y}$ & The actual label and the noisy label, respectively. \\
$p(x,y)$ & The joint distribution of data and labels. \\
$\widetilde{p}(x,y)$ & The joint distribution of data and labels after corruption by label noise. \\
$p(\widetilde{y}\mid y, x)$ & The noise model generating noisy labels $\widetilde{y}$ from clean labels $y$ given $x$. \\
$T$ & The label noise transition matrix describing the probabilities of transforming a true label into a noisy label. \\
$\bm{e}_k$ & The standard basis vector in $\mathbb{R}^c$ where only the $k$\textsuperscript{th} element is 1, and all others are 0.
\end{longtable}

\subsection{Label Noise Taxonomy}\label{sec:label_noise_tax}
When proving results about label-noise robust algorithms one typically needs to make some assumptions about the properties of the label noise. This necessitates a label noise taxonomy. Below we provide a summary of some of the common ways in which label noise is categorised. 

\paragraph*{Uniform Label Noise} Label noise is classified as \emph{uniform} or \emph{class-conditional} when $T(x)$ is constant across $\mathcal{X}$, denoted as $T$. Equivalently; $p(\widetilde{y}\mid y,x)=p(\widetilde{y}\mid y)$.

\paragraph*{Symmetric/Asymmetric Label Noise} Label noise is called \emph{symmetric} if all off-diagonal elements of $T$ are equal, indicating uniform mislabeling across classes. The transition matrix for symmetric label noise in the case of three classes ($c=3$) is shown in Equation~\ref{eqn:sym_noise}. Conversely, \emph{asymmetric} noise occurs when mislabeling probabilities vary among classes.

\begin{align}\label{eqn:sym_noise}
    \underbrace{\begin{pmatrix}
        1-\eta & \frac{\eta}{2} & \frac{\eta}{2}\\
        \frac{\eta}{2} & 1-\eta &  \frac{\eta}{2}\\
        \frac{\eta}{2} & \frac{\eta}{2} & 1-\eta 
    \end{pmatrix}}_{\text{Symmetric Label Noise}}
\end{align}

\paragraph*{Pairwise Label Noise:} An important type of asymmetric label noise is pairwise label noise. Given a classification task with $c$ classes and a transition matrix $T$, if for any class $i$, there exists at most one class $j \neq i$ with $T_{ij}=\eta$ and at most one class $k \neq i$ with $T_{ki}=\eta$, the noise is \emph{pairwise}. This mislabelling occurs between specific class pairs.

\paragraph*{Circular Label Noise:} A special case of pairwise noise occurs when, with probability $1-\eta$ labels remain uncorrupted and with probability a label transitions to the next class $j\mapsto j+1$, with the final class wrapping back around $c\mapsto 1$. This leads to transition matrix $T$ with a structure so that $T_{ii}=1-\eta$ and $T_{i,i-1}=\eta$ and $T_{1,c}=\eta$ (See Equation~\ref{eqn:circ}). We say that label noise of this type is \emph{circular}.

\begin{equation}\label{eqn:circ}
T = 
\underbrace{\begin{pmatrix}
1-\eta & 0 & \cdots & 0 & \eta \\
\eta & 1-\eta & \cdots & 0 & 0 \\
0 & \eta & \ddots & \vdots & \vdots \\
\vdots & \vdots & \ddots & 1-\eta & 0 \\
0 & 0 & \cdots & \eta & 1-\eta
\end{pmatrix}}_{\text{Circular Label Noise}}
\end{equation}

 \paragraph*{Diagonally Dominant} We say that $T$ is \emph{diagonally dominant} (DD) if for each $i$, the diagonal entry is greater than any other entry in its column $T_{ii} > \max_{j\neq i}T_{ji}$\footnote{As mentioned in \citep{anchor} this definition, commonly used in the context of noisy labels \citep{self_ensemble}, differs from the definition of DD from linear algebra.}. 

\paragraph*{Remark:} Our definition of diagonal dominance (DD) may initially seem to deviate from those presented in \citep{anchor} and \citet{xu2019l_dmi}. However, these references employ the convention where \( T_{ij} \) represents the probability of transitioning from \( y=i \) to \( \widetilde{y}=j \), which contrasts with our usage where \( T_{ij} \) indicates the probability of transitioning from \( y=j \) to \( \widetilde{y}=i \). This alternate convention simplifies the linear algebra involved. When this difference in conventions is considered, our definition aligns with those found in the literature.

\section{Further Related Work}\label{ch7:sec:further_related_work}
The following section builds upon Section~\ref{ch7:sec:related_work} by summarising additional methods for learning in the presence of noisy labels.

\subsection{Robust Loss Functions}
In Section~\ref{ch7:sec:related_work} we discussed robust loss functions, briefly mentioning Generalised Cross-Entropy (GCE) \citep{GCE_Loss} and forward-corrected loss functions. A great many robust loss functions have been proposed for learning in the presence of noisy labels. Beyond GCE, Symmetric Cross-Entropy \citep{sce} and Taylor Cross-Entropy \citep{taylorce} offer alternative methods for interpolating between CE, which tends to overfit, and MAE, which tends to underfit. In addition to forward-corrections, there are `backward-corrections' \citep{old_backward, fprop}, which correct the learning objective by denoising the labels using the estimated transition matrix. `Noise-Tolerant' loss functions \citep{robust_theory, robust_theory_earlier_binary, unhinged, ghosh2015making} seek loss functions that are innately robust, without the need for correction. `Normalised Loss Functions' \citep{normalised_losses} achieve robustness by normalising the loss functions. Peer Losses \citep{peer} introduce additional `peer' loss terms computed on noisy data to de-noise the loss objective. Other approaches regularise the loss objective to enhance robustness \citep{elr, flooding, gjs}, among others.

\subsection{Regularisation}
Neural network classifiers are prone to overfitting on noisy data, leading to the development of various regularisation strategies. \citet{memorisation} examines dropout \citep{dropout}, highlighting its partial effectiveness in mitigating overfitting. Data augmentation techniques, such as introducing feature noise during training, have shown promise in enhancing generalisation with noisy labels. For example, \citet{mixup} presents `MixUp,' which linearly combines data points and their labels. Approaches like `AutoAugment' \citep{cubuk2019autoaugment} and `RandAugment' \citep{cubuk2020randaugment} use randomised transformations to augment images. \citet{AugmentNLN} introduces `Augmented Descent' (AUGDESC) within the `DivideMix \citep{dividemix} framework, applying mild data transformations for loss modeling and stronger ones to improve generalisation. Although many of these methods were not initially intended for training with label noise, they have proven effective in this context \citep{noisyAugmentStudy}. Additionally, \citet{wei2021open} enhances robustness to intrinsic label noise by incorporating open-set noisy examples during training. Other regularisation strategies focus on model weights, such as weight decay \citep{surveyDeep}, gradient clipping \citep{gradientClip}, and Lipschitz constraints to limit model complexity, as discussed by \citet{pmlr-v119-harutyunyan20a} and \citet{gouk2021regularisation}.

\subsection{Architecture}
 Several methods address label noise by modifying architectures to improve robustness, often through adding a `noise adaption' layer. \citet{noiseAdapt2} introduces two approaches: a `bottom-up' method, which modifies model outputs using the estimated noise model, and a top-down' method, which reweights loss terms (though this is not an architectural change). The bottom-up approach estimates the noise transition matrix and attaches it as a final layer after the softmax. Similarly, \citet{noiseAdapt} adds a noise-adaption layer after the softmax, first training the network without the layer, then training both concurrently, initializing the layer with the confusion matrix. The `$s$-model' handles class-conditional noise, while the `$c$-model' accounts for instance-dependent noise. \citet{fprop_old} also uses a noise adaption layer but directly parameterises the transition matrix, applying trace regularisation to avoid it becoming too close to the identity. Earlier, \citet{mnihHinton} used a similar approach in the binary label setting, learning the noise layer via the EM algorithm \citep{dempster1977maximum}. Many of these methods relate closely to forward corrections, which we discussed in Section~\ref{ch7:sec:related_work}.

\subsection{Multi-Network Approaches}
 Numerous deep-learning methods for handling label noise are complex, often requiring multiple networks and stages in their pipelines \citep{dividemix, coteaching, mentornet, decoupling, EvidenceMix}. \citet{EvidenceMix} introduces EvidenceMix,' where two networks are initially trained for a few epochs on a dataset with both closed-set and open-set label noise. A Gaussian mixture model (GMM) then identifies and excludes open-set samples based on loss values. The networks subsequently continue training using semi-supervised learning (SSL) techniques on the remaining data. Other dual-network approaches operate similarly. For instance, `Co-teaching' \citep{coteaching} involves each network being trained on the lowest loss outputs of the other. `Decoupling' \citep{decoupling} updates the networks based on their disagreements. `MentorNet' \citep{mentornet} employs a teacher network to guide a student network by re-weighting samples that are likely correct. In `DivideMix' \citep{dividemix}, a Gaussian mixture model selects clean samples based on their loss values, which are then used to train the other network.

\section{Theoretical Proofs}\label{ch7:sec:proofs}
\subsection{Fact 3 - Bayes-optimality}\label{ch3:sec:fact3}
In this section we give formal statements of the facts listed in Section~\ref{ch7:sec:theoretical} along with proofs. We begin by demonstrating the veracity of Fact 3.

\begin{theorem}\label{ch3:thm:early_stopping_class_pres}
Let $p(x,y)$ be a data-label distribution and suppose that $\widetilde{p}(x,y)$ is a noisy version corrupted by class-preserving label noise. A probability estimator $\bm{q}^{*}$ is Bayes-optimal for the noisy distribution if and only if it is Bayes-optimal for the clean distribution. Equivalently;
\begin{align*}
    \bm{q}^* \in \argmin_{q}R_{0\text{-}1}(q) \iff \bm{q}^* \in \argmin_{q}R^{\eta}_{0\text{-}1}(q)
\end{align*}
\end{theorem}
\begin{proof}
This follows immediately from the definition of class-preserving, indeed, the definition of class-preserving is constructed precisely as the weakest noise condition for which this theorem holds. Specifically, $\bm{q}^*$ is a global minimiser of the clean risk if and only if, for every $x\in \supp(p(x))$, 
   \begin{align*}
      \argmax_{i\in\{1,2,\ldots c\} }\bm{q}^*_i(x) &= \argmax_{i\in\{1,2,\ldots c\} } p(y=i\mid x) 
    \end{align*}By the definition of class-preserving the right-hand side is equal to $\argmax_{i\in\{1,2,\ldots c\} } \widetilde{p}(\widetilde{y}=i\mid x)$ and hence 
    \begin{align*}
           \argmax_{i\in\{1,2,\ldots c\} }\bm{q}^*_i(x) &= \argmax_{i\in\{1,2,\ldots c\} } \widetilde{p}(\widetilde{y}=i\mid x)
    \end{align*}
    meaning that $\bm{q}*$ is a global minimiser of the noisy risk. 
\end{proof}

\subsection{Facts 1 and 4}\label{ch7:sec:fact1,4}
We establish the following Lemma relating the noisy and clean risk of an estimator in terms of the covariance between the estimator and the noise model. Facts 1 and 4 (Section~\ref{ch7:sec:theoretical}) then follow as corollaries. 

\begin{lemma}\label{ch7:lemma:risk_inequality}
    Let $\widetilde{p}(x,y)$ be some noisy data-label distribution corrupted by class-preserving (possibly non-uniform) symmetric label noise. Let $\bm{q}$ be an arbitrary probability estimator model and let $f$ be its plug-in classifier. 
    Let 
    \begin{align*}
        g(x) \coloneq p(y=f(x)\mid x)
    \end{align*}
    This function gives the probability of our predicted class appearing at $x$. The expectation of $g(x)$ gives the proportion of labels predicted corrected (i.e. accuracy). Let $\sigma_g$ denote the standard deviation of $g(x)$. Let $\eta(x)$ denote the noise rate at $x\in\mathcal{X}$ and let $\sigma_{\eta}$ denote the standard deviation of $\eta(x)$ with respect to $p(x)$.  The noisy risk and clean risks of $\bm{q}$ are related via 
        \begin{align*}
          R^{\eta}(\bm{q})  = R(\bm{q})\left(1-\frac{c\eta}{c-1}\right) + \eta + \frac{c}{c-1}\text{Cov}(g(x), \eta_x)
    \end{align*}
    which leads to the inequality
    \begin{align*}
         R(\bm{q})\left(1-\frac{c\eta}{c-1}\right)  + \eta - \frac{\sigma_{\eta}\sigma_{g}c}{c-1} &\leq R^{\eta}(\bm{q})  \leq R(\bm{q})\left(1-\frac{c\eta}{c-1}\right) + \eta + \frac{\sigma_{\eta}\sigma_{g}c}{c-1}
    \end{align*}

    Thus, in particular when the noise is uniform $\sigma_{\eta}=0$ one has
    \begin{align*}
        R^{\eta}(\bm{q}) = R(\bm{q})\left(1-\frac{c\eta}{c-1}\right) +\eta
    \end{align*}
\end{lemma}
\begin{proof}

    Let $\bm{q}$ be a probability estimator and let $f:\mathcal{X}\rightarrow \{1,2,\ldots ,c \}$ be the associated plug-in classifier. The clean and noisy $0\text{-}1$ risks of $\bm{q}$ can be expressed as
    \begin{align*}
        R(\bm{q}) &= 1- \mathbb{E}_{x,y\sim p(x,y)}[p(y=f(x)\mid x)]\\
        R^{\eta}(\bm{q}) &= 1- \mathbb{E}_{x,\widetilde{y}\sim \widetilde{p}(x,y)}[p(\widetilde{y}=f(x)\mid x)].
    \end{align*}
    We assume that the label noise is (non-uniform) symmetric label noise meaning that, at each $x$, the $p(\widetilde{y}\mid y)$ may be expressed by a matrix $T(x)$ of the following form
    \begin{align*}
\begin{bmatrix}
1 - \eta_x & \frac{\eta_x}{c-1} & \frac{\eta_x}{c-1} & \cdots & \frac{\eta_x}{c-1} \\
\frac{\eta_x}{c-1} & 1 - \eta_x & \frac{\eta_x}{c-1} & \cdots & \frac{\eta_x}{c-1} \\
\frac{\eta_x}{c-1} & \frac{\eta_x}{c-1} & 1 - \eta_x & \cdots & \frac{\eta_x}{c-1} \\
\vdots & \vdots & \vdots & \ddots & \vdots \\
\frac{\eta_x}{c-1} & \frac{\eta_x}{c-1} & \frac{\eta_x}{c-1} & \cdots & 1 - \eta_x.
\end{bmatrix}
\end{align*}
    This allows us to write
\begin{align}
        \mathbb{E}_{x,\widetilde{y}\sim \widetilde{p}(x,y)}[p(\widetilde{y}=f(x)\mid x)] &= \mathbb{E}_{x,y\sim p(x,y)}[\mathbb{E}_{\widetilde{y}\sim p(\widetilde{y}\mid x,y)}[p(\widetilde{y}=f(x)\mid x)] ] \\
       &= \mathbb{E}_{x \sim p(x)}[\left(T(x)\bm{p}(y\mid x)\right)_{f(x)}] \label{ch3:eqn:symetric_nonuniform_lemma}
\end{align}

For any probability vector $\bm{p}$, the $k$\textsuperscript{th} component of $T(x)\bm{p}$ is equal to 
\begin{align*}
    (1-\eta_x)p_k + \frac{\eta_x}{c-1}\sum_{i\neq k}p_i &= (1-\eta_x)p_k + (1-p_k)\frac{\eta_x}{c-1}\\
    &= p_k\left(1-\eta_x - \frac{\eta_x}{c-1} \right) + \frac{\eta_x}{c-1}
\end{align*}
Thus, using Equation~\ref{ch3:eqn:symetric_nonuniform_lemma}, the noisy risk can be written as 
\begin{align}
    R^{\eta}(\bm{q}) = 1- \mathbb{E}_{x \sim p(x)}\left[p(y=f(x)\mid x)\left(1-\eta_x - \frac{\eta_x}{c-1} \right) + \frac{\eta_x}{c-1}\right]\notag \\
    = R(\bm{q}) - \frac{\eta}{c-1} + \mathbb{E}_{x \sim p(x)}\left[p(y=f(x)\mid x)\left(\frac{c\eta_x}{c-1} \right)\right] \label{ch3:eqn:to_bound}
\end{align}

Using $g(x)\coloneq p(y=f(x)\mid x)$ for brevity and identifying that 
\begin{align*}
    \mathbb{E}_{x \sim p(x)}\left[g(x)\eta_x\right] = \text{Cov}(g(x), \eta_x) +\mu_g \mu_{\eta},
\end{align*}
we can bound the final term of Equation~\ref{ch3:eqn:to_bound}. Hence after rearranging we arrive at our first claim:
\begin{align*}
  R^{\eta}(\bm{q})  = R(\bm{q})\left(1-\frac{c\eta}{c-1}\right) + \eta + \frac{c}{c-1}\text{Cov}(g(x), \eta_x).
\end{align*}

Generally we cannot expect that $g(x)$ will be independent of $\eta_x$. Using Cauchy-Schwarz on the random variables $X\coloneq \eta_x-\eta$, $Y\coloneq g(x)-\mu_g$ one obtains the following bound on the covariance 
\begin{align*}
    \vert \text{Cov}(g(x), \eta_x)\vert \leq \sqrt{\text{Var}(g(x))\text{Var}(\eta_x)} 
\end{align*}
Denoting the standard deviations of $\eta_x$ and $g(x)$ as $\sigma_{\eta}, \sigma_{g}$ respectively we have 
\begin{align*}
    \sigma_{\eta}\sigma_{g} - \mu_{\eta}\mu_{g} \leq  \mathbb{E}_{x \sim p(x)}\left[g(x)\eta_x\right] \leq \sigma_{\eta}\sigma_{g} + \mu_{\eta}\mu_{g}
\end{align*}

Thus we have the following upper and lower bounds on the noise risk
\begin{align*}
   R(\bm{q}) - \frac{\eta}{c-1}+  \frac{c}{c-1}\left( - \sigma_{\eta}\sigma_{g} + \mu_{\eta}\mu_{g}\right) \leq R^{\eta}(\bm{q}) \leq R(\bm{q}) - \frac{\eta}{c-1} + \frac{c}{c-1}\left(\sigma_{\eta}\sigma_{g} + \mu_{\eta}\mu_{g}\right)
\end{align*}

$\mu_g$ is precisely the accuracy of the estimator $\bm{q}$ and thus $R(\bm{q}) = 1-\mu_g$, likewise $\mu_{\eta}$ is the mean noise rate $\eta$. Putting this together we obtain
\begin{align}
     R(\bm{q}) - \frac{\eta}{c-1} + \frac{c\eta}{c-1}(1-R(\bm{q})) - \frac{\sigma_{\eta}\sigma_{g}c}{c-1}  &\leq R^{\eta}(\bm{q}) \leq R(\bm{q}) - \frac{\eta}{c-1} +  \frac{c\eta}{c-1}(1-R(\bm{q})) + \frac{\sigma_{\eta}\sigma_{g}c}{c-1} \notag \\
     R(\bm{q})\left(1-\frac{c\eta}{c-1}\right)  + \eta - \frac{\sigma_{\eta}\sigma_{g}c}{c-1} &\leq R^{\eta}(\bm{q})  \leq R(\bm{q})\left(1-\frac{c\eta}{c-1}\right) + \eta + \frac{\sigma_{\eta}\sigma_{g}c}{c-1} \label{ch3:eqn:full_bound}
\end{align}

Which is precisely our second claim. Finally, note the two special cases:

\textbf{Uniform: }    If $\eta_x = \eta = const.$ then this becomes
\begin{align*}
    R(\bm{q}) - \frac{\eta}{c-1} +  (1-R(\bm{q}) )\left(\frac{c\eta}{c-1} \right) \\
    =R(\bm{q})\left(1-\frac{c\eta}{c-1}\right) +\eta
\end{align*}
    as claimed.

\textbf{Independence:} In the event that $g(x)$ and $\eta_x$ are uncorrelated we have
\begin{align*}
    R^{\eta}(\bm{q}) &=R(\bm{q}) - \frac{\eta}{c-1} +\frac{c}{c-1} \mu_g\mu_{\eta}\\ &= R(\bm{q}) - \frac{\eta}{c-1} +\frac{c}{c-1}(1-R(\bm{q}))\eta\\
    &= R(\bm{q})\left(1-\frac{\eta c}{c-1}\right) + \eta
\end{align*}
\end{proof}

\paragraph*{Fact 1: Symmetric Uniform Noise} Lemma~\ref{ch7:lemma:risk_inequality} establishes Fact 1 in Section~\ref{ch7:sec:theoretical}: that when the label noise is symmetric, uniform and class-preserving ($\eta<\frac{c-1}{c}$) the noisy and clean risk are related by a linear map. This is significant because it means that if $\bm{q}_1,\bm{q}_2$ are two models then $R(\bm{q}_1)\leq R(\bm{q}_2) \iff R^{\eta}(\bm{q}_1)\leq R^{\eta}(\bm{q}_2)$: If one has a set of probability estimators $\{\bm{q}_n\}_{n=1}^N$
 and we select $\bm{q}_k$ which minimises the noisy risk then this will also minimise the clean risk. 

 \paragraph*{Fact 4:} Lemma~\ref{ch7:lemma:risk_inequality} also establishes Fact 4, that there is an affine relationship between the clean and noisy risk of a probability estimator when the noise model $\eta_x$ and the model's accuracy function ($g(x) \coloneq p(y=f(x)\mid x)$ where $f$ is the plug-in classifier of the estimator) are uncorrelated (assuming non-uniform symmetric label noise). Thus, if one selects the minimiser of the noisy risk among a set of estimators $\mathcal{Q}\coloneq \{\bm{q}_n\}_{n=1}^N$, this model necessarily is also a minimiser of the clean risk. 

\subsection{Fact 2}\label{ch3:sec:fact2}
 We have shown in Lemma~\ref{ch7:lemma:risk_inequality} that, when label noise is uniform and symmetric and $\eta<\frac{c-1}{c}$ then for \emph{any} two estimators $\bm{q}_1,\bm{q}_2$ and \emph{any} data-label distribution $p(x,y)$ 
 \begin{align}\label{ch3:eqn:risk_implication}
     R^{\eta}(\bm{q}_1) \leq R^{\eta}(\bm{q}_2) \iff R(\bm{q}_1) \leq R(\bm{q}_2)
 \end{align}
We now endeavour to show the converse, that this only holds for uniform symmetric noise below the specified threshold. 

Assume we have some closed-set label noise model for which Equation~\ref{ch3:eqn:risk_implication} holds for all distributions $p(x,y)$ and estimators $\bm{q}_1, \bm{q}_2$.  Since Equation~\ref{ch3:eqn:risk_implication} holds for all distributions then it holds in particular when we set $p(x)$ to be a Dirac delta distribution at some point $x\in\mathcal{X}$, for all possible conditional distributions $\bm{p}(y\mid x)\in \Delta$. Given two estimators $\bm{q}_1, \bm{q}_2$ we let $k_1\coloneq \argmax_i (\bm{q}_1(x))_i$, $k_2\coloneq \argmax_i (\bm{q}_2(x))_i$ denote their predicted labels at $x$. Then for any pair of predicted labels $k_1,k_2\in\mathcal{Y}$, and for any $\bm{p}(y\mid x)\in \Delta$, Equation~\ref{ch3:eqn:risk_implication} implies
 \begin{align*}
\left( 1-\widetilde{p}(y=k_1\mid x) \leq 1-\widetilde{p}(y=k_2\mid x) \right) & \iff \left( 1-p(y=k_1\mid x) \leq 1-p(y=k_2\mid x) \right).
 \end{align*}
 This can be written equivalently as
\begin{align}\label{ch3:eqn:risk_implication2}
    \left( \widetilde{p}(y=k_2\mid x) \leq \widetilde{p}(y=k_1\mid x) \right) & \iff \left( p(y=k_2\mid x) \leq p(y=k_1\mid x) \right).
\end{align}
Since we are modelling closed-set noise we know that the noise can be modelled by a transition matrix $T$ at $x$. Thus, letting $\bm{p}$ be shorthand for the condition distribution at $x$; $\bm{p}\coloneq \bm{p}(y\mid x)$, Equation~\ref{ch3:eqn:risk_implication2} is equivalent to
\begin{align}\label{ch3:eqn:pointwise_iff}
    (T\bm{p})_{k_2} \leq (T\bm{p})_{k_1}  & \iff  \bm{p}_{k_2} \leq \bm{p}_{k_1} .
\end{align}

Our goal is to demonstrate that, in order for Equation~\ref{ch3:eqn:pointwise_iff} to hold $\forall \bm{p}\in \Delta$ and $\forall k_1, k_2 \in \mathcal{Y}$, that $T$ must describe symmetric label noise. Equation~\ref{ch3:eqn:pointwise_iff} implies that  $(T\bm{p})_{k_2} = (T\bm{p})_{k_1} \iff \bm{p}_{k_2} = \bm{p}_{k_1}$ since $(\bm{p}_{k_2}\leq \bm{p}_{k_1}) \land (\bm{p}_{k_1}\leq \bm{p}_{k_2}) \implies \bm{p}_{k_1}=\bm{p}_{k_2}$. The first major implication of this is that $T$ must be row stochastic as well as columns stochastic - rows sum to one. To see this let $\bm{p}= \left(\frac{1}{c}, \frac{1}{c}, \ldots, \frac{1}{c}\right)$; the uniform distribution over labels. Then $T\bm{p}$ must also be the uniform distribution, thus, the rows of $T$ must all sum to $1$ (since $(T\bm{p})_i = \frac{1}{c} \iff \sum_{j}T_{ij}\frac{1}{c} = \frac{1}{c}\iff \sum_{j}T_{ij} = 1$). 

Now letting $\bm{p}=\bm{e}_k$ (the $k$\textsuperscript{th} coordinate vector), Equation~\ref{ch3:eqn:pointwise_iff} implies that $(T\bm{p})_i = (T\bm{p})_j $ for all $i,j\neq k$. As $T\bm{p}$ is the $k$\textsuperscript{th} column of $T$ then for all $i,j\neq k$, $T_{ik} = T_{jk} $. This allows us to write $T$ as

\begin{align*}
    \begin{bmatrix}
        T_{11} & a_2 & a_3 & \ldots & a_c\\
        a_1 & T_{22} & a_3 & \ldots & a_c\\
        \ldots\\
        a_1 & a_2 & a_3 & \ldots & T_{cc}
    \end{bmatrix}
\end{align*}
Since we know that the matrix is column and row stochastic then, in particular, the $k$\textsuperscript{th} row and column sum to one so
\begin{align*}
    T_{kk} + \sum_{i\neq k} a_i = T_{kk} + (c-1)a_k = 1\\
    \iff \sum_{i=1}^c a_i = c a_k
\end{align*}
We can write this as a system of equations
\begin{align}
    \begin{bmatrix}
        1-c & 1 & 1 & \ldots & 1\\
    1 & 1-c & 1 & \ldots & 1\\
    \ldots\\
    1 & 1 & 1 & \ldots & 1-c
    \end{bmatrix}
    \begin{bmatrix}
        a_1 \\ a_2 \\ \ldots \\ a_c
    \end{bmatrix}
    = \bm{0}
\end{align}
Thus the $\bm{a}$ must lie in the null-space of this matrix which may be computed. We identify that all vectors of the form $(\lambda, \lambda, \ldots, \lambda)$ lie in the nullspace. However, as the rank of the matrix is $c-1$ then by the rank-nullity theorem this is the entire nullspace. 

Hence we know that $a_1=a_2=\ldots = a_c\eqcolon a$ meaning that 
\begin{align*}
    T =     \begin{bmatrix}
        T_{11} & a & a & \ldots & a\\
        a & T_{22} & a & \ldots & a\\
        \ldots\\
        a & a & a & \ldots & T_{cc}
    \end{bmatrix}
\end{align*}
By the condition that the rows and columns sum to one this must be writable as 
\begin{align*}
    T =     \begin{bmatrix}
        1-a(c-1) & a & a & \ldots & a\\
        a & 1-a(c-1) & a & \ldots & a\\
        \ldots\\
        a & a & a & \ldots & 1-a(c-1)
    \end{bmatrix}
\end{align*}
then letting $a \coloneq \frac{\eta}{c-1}$ we have
\begin{align*}
    T =     \begin{bmatrix}
        1-\eta & \frac{a}{c-1} & \frac{\eta}{c-1} & \ldots & \frac{\eta}{c-1}\\
        \frac{\eta}{c-1} & 1-\eta & \frac{\eta}{c-1} & \ldots & \frac{\eta}{c-1}\\
        \ldots\\
        \frac{\eta}{c-1} & \frac{\eta}{c-1} & \frac{\eta}{c-1} & \ldots & 1-\eta
    \end{bmatrix}
\end{align*}
which is precisely the matrix describing symmetric label noise as desired. 

Thus for Equation~\ref{ch3:eqn:risk_implication2} to hold for all estimators and data-label distributions the label noise must be symmetric at every $x$. It remains to show that the noise rate must be uniform. This may be derived using Lemma~\ref{ch7:lemma:risk_inequality} which states that for non-uniform symmetric label noise
        \begin{align*}
          R^{\eta}(\bm{q})  = R(\bm{q})\left(1-\frac{c\eta}{c-1}\right) + \eta + \frac{c}{c-1}\text{Cov}(g(x), \eta_x).
    \end{align*}
Where $g(x)\coloneq p(f(x)\mid x)$ gives the probability of the predicted label at $x$. It follows, that for any two estimators $\bm{q}_1, \bm{q}_2$ 
\begin{align*}
          R^{\eta}(\bm{q}_1) &\leq R^{\eta}(\bm{q}_2) \iff \\
          R(\bm{q}_1) &\leq R(\bm{q}_2) + \frac{1}{\frac{c-1}{c}-\eta}\left(\text{Cov}(g_2(x), \eta_x) - \text{Cov}(g_1(x), \eta_x\right).
\end{align*}
To ensure the difference between the covariances vanishes for all distributions $p(x,y)$ and estimators $\bm{q}_1, \bm{q}_2$ we must have $\eta_x = \eta = const.$ Thus the label noise model must be uniform symmetric noise as claimed. This establishes Fact 2.

\subsection{Fact 5}\label{ch3:sec:worst_case_proofs}
With Fact 2 we saw that for any noise model other than uniform symmetric label noise, the minimiser of the noisy risk may not be a minimiser of the clean risk. In this section, we provide some worst-case bounds. Specifically, if 
\begin{align*}
    \bm{q}^{\eta}_* &\coloneq \argmin_{\bm{q}\in \mathcal{Q}} R^{\eta}(\bm{q})\\
    \bm{q}_* &\coloneq \argmin_{\bm{q}\in \mathcal{Q}} R(\bm{q})
\end{align*}
denote the minimisers of the noisy and clean risks within $\mathcal{Q}$. Then we seek to bound
\begin{align*}
    \vert R(\bm{q}^{\eta}_*) - R(\bm{q}_*)\vert.
\end{align*}
To simplify the mathematics make some assumptions. We assume that the data distribution is separable and we suppose that the every column of the transition matrix is a permutation of every other column.

\begin{theorem}\label{ch7:thm:worst_case_bounds}
Consider a scenario with class-conditional label noise characterised by a transition matrix $T$. Define $\eta_{max}$ as the maximum transition probability, i.e., $\eta_{max} \coloneqq \max_{j\neq i} T_{ij}$, and $\eta_{min}$ as the minimum transition probability, i.e., $\eta_{min} \coloneqq \min_{j\neq i} T_{ij}$. Let $\mathcal{Q} \coloneqq \{\bm{q}\}_{i=1}^N$ be a set of probability estimators. Denote $\bm{q}^{\eta}_*$ as the minimiser of the noisy risk and $\bm{q}_*$ as the minimiser of the clean risk, respectively:
\begin{align*}
\bm{q}^{\eta}_* &\coloneqq \argmin_{\bm{q}\in \mathcal{Q}} R^{\eta}(\bm{q}),\\
\bm{q}_* &\coloneqq \argmin_{\bm{q}\in \mathcal{Q}} R(\bm{q}).
\end{align*}
Denote the clean and noisy risks of these estimators as:
\begin{align*}
R_k \coloneqq R(\bm{q}^{\eta}_*)\quad \text{and}\quad R^{\eta}_* \coloneqq R(\bm{q}^{\eta}_*),\\
R_* \coloneqq R(\bm{q}_*)\quad \text{and}\quad R^{\eta}_l \coloneqq R(\bm{q}_*).
\end{align*}
The difference between the optimal clean risk $R_*$ and the clean risk achieved by $\bm{q}^{\eta}_*$ satisfies the following inequality:
\begin{align*}
|R_k - R_*| \leq \frac{R^{\eta}_* - \eta}{1 - \eta - \eta_{max}} - \frac{R^{\eta}_l - \eta}{1 - \eta - \eta_{min}}.
\end{align*}
Thus, given that $R_*^{\eta}$ is optimal, we have:
\begin{align}\label{ch7:eqn:worst_case_perm_sym}
|R_k - R_*| &\leq (R^{\eta}_* - \eta)\left(\frac{1}{1 - \eta - \eta_{max}} - \frac{1}{1 - \eta - \eta_{min}}\right).
\end{align}
\end{theorem}

\begin{proof}
     Let $\bm{q}_j\in \mathcal{Q}$ be an arbitrary estimator in our set. Let $A^+\subseteq \supp(p(x))$ denote the set upon which $\bm{q}_j$ correctly predicts the true label and we use $A^-$ to denote the complement of $A^+$ in $\supp(p(x))$ so that $A^+ \cup A^- = \supp(p(x))$. We let $f(x)$ denote the plug-in classifier induced by $\bm{q}_j$.  By the separability assumption, we know that 
     \begin{align*}
         p(y=f(x)\mid x)=1\quad  \text{for} \quad x\in A^+,\\
         p(y=f(x)\mid x)=0\quad  \text{for} \quad x\not\in A^+.
     \end{align*}
     Likewise, 
     \begin{align*}
         \widetilde{p}(\widetilde{y}=f(x)\mid x)&=1-\eta\quad \text{for} \quad x\in A^+,\\
         \widetilde{p}(\widetilde{y}=f(x)\mid x)&\in [\eta_{min}, \eta_{max}]\quad\text{for} \quad x\not\in A^+.
     \end{align*}
     This allows us to write the following inequality lower bounding the noisy risk of $\bm{q}_j$; 
     \begin{align*}
         R^{\eta}_j &= \int_{A^+} p(x)(1-\widetilde{p}(\widetilde{y}=f(x)\mid x))dx +  \int_{A^-} p(x)(1-\widetilde{p}(\widetilde{y}=f(x)\mid x))dx\\
         &\geq \mu(A^+)\eta + \mu(A^-)(1-\eta_{max}).
     \end{align*}
     Here $\mu$ denotes the measure associated with the density $p$, so that $\int_{A^+}p(x)dx = \mu(A^+)$.
     We can similarly deduce that
          \begin{align*}
         R^{\eta}_j \leq \mu(A^+)\eta + \mu(A^-)(1-\eta_{min}) .
     \end{align*}
     We note that $\int_{A^+}p(x)dx = \mu(A^+) = 1-R_j$ thus we have the following inequality between the noisy and clean risks of an arbitrary estimator $\bm{q}_j\in\mathcal{Q}$
     \begin{align}\label{ch3:eqn:worst_case_risk_bounds1}
         R^{\eta}_j &\leq (1-R_j)\eta + R_j(1-\eta_{min}) = R_j(1-\eta_{min}-\eta) +\eta,\\
        R^{\eta}_j &\geq (1-R_j)\eta + R_j(1-\eta_{max}) = R_j(1-\eta_{max}-\eta) +\eta.\label{ch3:eqn:worst_case_risk_bounds2}
     \end{align}
    Rearranging this is equivalent to
    \begin{align*}
       \frac{R_j^{\eta}-\eta}{1-\eta_{min}-\eta} \leq R_j \leq \frac{R_j^{\eta}-\eta}{1-\eta_{max}-\eta}
    \end{align*}
    Hence, in particular
    \begin{align*}
        \frac{R_l^{\eta}-\eta}{1-\eta_{min}-\eta} \leq R_* \leq \frac{R_l^{\eta}-\eta}{1-\eta_{max}-\eta},\\
        \frac{R_*^{\eta}-\eta}{1-\eta_{min}-\eta} \leq R_k \leq \frac{R_*^{\eta}-\eta}{1-\eta_{max}-\eta}.
    \end{align*}
Putting these together we conclude
    \begin{align*}
        \frac{R^{\eta}_l-\eta}{1-\eta-\eta_{min}} &\leq R_* \leq R_k \leq \frac{R^{\eta}_*-\eta}{1-\eta-\eta_{max}} \\
        \implies \vert R_k - R_*\vert &\leq \frac{R^{\eta}_*-\eta}{1-\eta-\eta_{max}} - \frac{R^{\eta}_l-\eta}{1-\eta-\eta_{min}}\\
        \implies \vert R_k - R_*\vert &\leq (R^{\eta}_*-\eta)\left(\frac{1}{1-\eta-\eta_{max}} - \frac{1}{1-\eta-\eta_{min}}  \right)
    \end{align*}
    as desired. Note that if $\eta_{min} = \eta_{max}$ (i.e. the label noise is symmetric) that $\vert R_k - R_*\vert=0$ which is consistent with Fact 1 (Section~\ref{ch7:sec:fact1,4}).
\end{proof}

\begin{corollary}\label{ch7:cor:bound}
    For Pairwise label noise (i.e. where $T_{11}=1-\eta$ and for some $i\neq 1$ $T_{1i} = \eta$, refer to Section~\ref{sec:label_noise_tax}) we have the following bound;
    \begin{align*}
        \vert R_k - R_* \vert \leq \dfrac{\eta \left(R^{\eta}_k - \eta\right) }{(1-2\eta)(1-\eta)}
    \end{align*}
    Note that for a noise rate $\eta$, a noisy risk below $\eta$ is unattainable and $R_k\rightarrow R_*$ as $R^{\eta}_k\rightarrow \eta$.
\end{corollary}

\subsubsection{Discussion}\label{ch7:sec:bounds_discussion}
\begin{align}\label{ch7:eqn:perm_sym}
T = \begin{bmatrix}
    0.5 & 0.2 & 0.3 \\
    0.3 & 0.5 & 0.2 \\
    0.2 & 0.3 & 0.5
\end{bmatrix}
\end{align}
Using some examples we can get a sense of how good the bounds we derived in Theorem~\ref{ch7:thm:worst_case_bounds} and Corollary~\ref{ch7:cor:bound}. Start by considering a three-class dataset corrupted by label noise described by the transition matrix in Equation~\ref{ch7:eqn:perm_sym}. Suppose we train a classifier and that during training we achieve a peak noisy accuracy of $40\%$. Plugging these numbers into Equation~\ref{ch7:eqn:worst_case_perm_sym} we get
\begin{align*}
     \vert R_k - R_* \vert \leq  \frac{1}{6}.
\end{align*}
This means that, in the worst case, the optimal clean test accuracy attained during training could be $16.6\%$ higher than the accuracy of our model selected by NES.

 As a second example, suppose that we have a dataset corrupted by pairwise with a noise rate of $\eta=40\%$, we train a neural network estimator on this dataset. The peak noisy validation accuracy attained during training is $50\%$. Then 
\begin{align*}
    \vert R_k - R_* \vert \leq \frac{1}{3} 
\end{align*}
Thus in this setting, the difference in clean test accuracy between the NES-chosen model and the optimal model could be as high as $33.3\%$!

Neither of the bounds in these examples are particularly tight: Within the context of neural network classifiers trained on image datasets, a decrease in test accuracy of about $10\%$ is large indeed. Empirically we find that a Noisy Early-Stopping policy generally permits us to attain a model which is within a single percentage point of optimal clean test accuracy. These bounds are therefore inadequate to explain why this performance is so good.

\subsection{Lemma~\ref{ch3:lemma:g_vector_lemma} and Generalisations}\label{ch3:sec:generalising_g_vector_lemma}
\begin{lemma}
    Suppose we train a classifier on a dataset corrupted by (class-preserving) label noise for $N$ epochs. Let $f^{(n)}$ denote the classifier obtained after training for $n$ epochs. Let $\bm{g}^{(n)} = (g^{(n)}_1, g^{(n)}_2, \ldots, g^{(n)}_c)$ denote the $g-$vector of the $n$\textsuperscript{th} classifier $f^{(n)}$. Suppose there exists an integer $T< N$ so that for all $i\geq 2$, $g^{(n)}_i$ is decreasing for $n\in\{1,2,\ldots, T\}$ and increasing for $n\in\{T,T+1\ldots, N\}$ then the minimiser of the noisy risk within $\{f^{(n)}\}_{n=1}^N$ also minimises the clean risk.
\end{lemma}
\begin{proof}
    Since we are assuming that the data-label distribution is separable we know that the clean conditional class distribution $\bm{p}(y\vert x)=\bm{e}_k$ for some $k$. It follows that, if $\bm{g}^{(n)}$ is the $g-$vector of $f^{(n)}$, then the clean accuracy of the $n$\textsuperscript{th} classifier $f^{(n)}$ is equal to $g_1^{(n)}$. We know that $\sum_{i=1}^c g_i^{(n)} = 1$ so clean accuracy can be expressed as 
    $1- g_2^{(n)}- g_3^{(n)} -\ldots -g_c^{(n)}$. Since by assumption each of the $g_i^{(n)}$ attain their minimum simultaneously at $n=T$ then the clean accuracy is maximised as $n=T$. Our goal now is to show that the noisy accuracy is also maximised at this epoch. 
    
    The noisy accuracy may be expressed
    \begin{align}
        \int p(x)\widetilde{p}(\widetilde{y}=f^{(n)}(x)\mid x)dx &= \int p(x)\left(T(x)\bm{p}(y\mid x)\right)_{f^{(n)}(x)}dx \notag\\
        =\int p(x)\left(T\bm{e}_{k(x)}\right)_{f^{(n)}(x)}dx &= \int p(x)\left(T_{k(x),f^{(n)}(x)}\right)dx. \label{eqn:nes_lemma_eqn1}
    \end{align}
    We are assuming that $T$ is permutation-symmetric which means that each of its columns are permutations of each other. Suppose that each column of $T$ is a permutation of the vector $(1-\eta, \eta_2, \ldots, \eta_c)$ where $\eta_i< \eta_j$ for $i<j$ and $\eta\coloneq \eta_2+\eta_3 + \ldots \eta_c$. Then Equation~\ref{eqn:nes_lemma_eqn1} can be written as 
    \begin{align}
       & \int p(x)T_{k(x),f^{(n)}(x)}dx \notag \\
        &= \sum_{i=2}^c\int p(x)\eta_i p(T_{k(x),f^{(n)}(x)} = \eta_i \mid x)dx +  \int p(x)(1-\eta) p(T_{k(x),f^{(n)}(x)} = 1-\eta\mid x)dx \notag\\
        &= \sum_{i=2}^c\eta_i \int p(x)p(T_{k(x),f^{(n)}(x)} = \eta_i\mid x)dx +  (1-\eta)\int p(x)p(T_{k(x),f^{(n)}(x)} = 1-\eta\mid x)dx \notag\\
        &= \sum_{i=2}^c\eta_i g_i^{(n)}  +  (1-\eta)g_1^{(n)} \label{eqn:nes_lemma_eq2}\\
        &= (1-\eta, \eta_2, \eta_3, \ldots, \eta_c)\cdot (g_1^{(n)}, g_2^{(n)}, \ldots, g^{(n)}_c) \notag
    \end{align}
    Using the fact that $\sum_i g_i = 1$ we know that Equation~\ref{eqn:nes_lemma_eq2} can be written as 
    \begin{align*}
        \sum_{i=2}^c\eta_i g_i^{(n)}  +  (1-\eta)g_1^{(n)} &= \sum_{i=2}^c\eta_i g_i^{(n)}  +  (1-\eta_2-\eta_3 \ldots -\eta_c)(1-g_2^{(n)}- g_3^{(n)}-\ldots-g_c^{(n)}) \\
        &= \left(\sum_{i=2}^c g_i^{(n)}(\eta_2+\ldots +2\eta_i+\ldots+\eta_c - 1)\right) + (1-\eta_2-\eta_3 \ldots - \eta_c).
    \end{align*}
    Since our label noise is class-preserving we know that $(1-\eta)>\eta_2>\ldots >\eta_c$ so for every $i\geq 2$, 
    \begin{align*}
        \eta_2+\ldots +2\eta_i+\ldots+\eta_c - 1 = \eta_i-(1-\eta)<0
    \end{align*}
So, the noisy accuracy at epoch $n$ may be expressed as
\begin{align*}
    A^{(n)}_{\eta} = \sum_{i=2}^c \alpha_i g_i^{(n)} + const.
\end{align*}
where $\alpha_i <0$. By assumption, each of the $g_i^{(n)}$ are decreasing for $n\in \{1,2,\ldots, T\}$ and therefore $A^{(n)}_{\eta}$ is increasing for $n\in \{1,2,\ldots, T\}$. Likewise the $g_i^{(n)}$ are increasing for $n\in \{T, T+1, \ldots, N\}$ meaning that the noisy accuracy is decreasing for $n\in \{T, T+1, \ldots, N\}$. Hence the noisy accuracy attains its maximum at $n=T$ - the same epoch as the clean accuracy. 
\end{proof}

\paragraph*{Asymmetric} We can generalise Lemma~\ref{ch3:lemma:g_vector_lemma} to the case of uniform asymmetric label noise fairly easily. We associate to each classifier a `$g-$matrix' rather than a vector. The entries of the matrix are the probabilities that the given classifier predicts class $i$ when the true class is $j$ - i.e. it is the confusion matrix for the classifier. The noisy accuracy is then given as a Hadamard product between the noise transition matrix and this confusion matrix.  Under the assumption that each of the $g_{ij}, i\neq j$ achieve their minima simultaneously the noisy and clean accuracy are maximised at the same training epoch. 

\paragraph*{Non-Uniform} Generalising Lemma~\ref{ch3:lemma:g_vector_lemma} to the case of non-uniform label noise is more non-trivial. Much of the proof of Lemma~\ref{ch3:lemma:g_vector_lemma} holds in the non-uniform case, however, the $\eta_i$ now vary during training. More significantly, if $\eta$ denotes the average noise rate on the subset of dataspace upon which our classifier predicts the clean label and $\eta_i$ similarly denotes the average noise rate on the subset of dataspace upon which our classifier predicts the $i$\textsuperscript{th} most likely noisy label, then we no longer have $\eta = \sum_{i>1}\eta_i$. This prevents us from concluding that the coefficients of each of the $g^{(n)}_i$ are negative without further assumptions. We add in the assumption that the classifier tends to predict more accurately (with respect to the clean distribution) in regions of dataspace with a low noise rate. This added assumption allows us to extend the proof to the non-uniform case. We do not write this out in full and may be seen as an exercise for the reader. 

\paragraph*{Non-Separable} We have assumed that the data distribution is separable. Handling the non-separable case is difficult - however, we can reason informally that Lemma~\ref{ch3:lemma:g_vector_lemma} should be extendable to the non-separable case by noting that label noise and aleatoric randomness are mathematically indistinguishable. Specifically, suppose that $p(x,y)$ is non-separable. This means that some of the clean class conditionals $\bm{p}(y\vert x)$ are not expressible as a coordinate vector $\bm{e}_k$ for some $k$. We can conceptualise the clean distribution $p(x,y)$ therefore as being a noised version of some separable distribution $p^{\circ}(x,y)$. The label noise associating these distributions would generally be non-uniform, asymmetric and must be class-preserving. Suppose now we have a noisy distribution $\widetilde{p}(x,y)$ which is obtained by applying label noise to $p(x,y)$ then, by concatenation of these two label noise processes, $\widetilde{p}(x,y)$ can be thought of as a noised version of $p^{\circ}(x,y)$. Our results tell us therefore that if we were to train a classifier on samples drawn from $\widetilde{p}(x,y)$, then the $\widetilde{p}(x,y)-$accuracy should peak at the same time as the $p^{\circ}(x,y)-$accuracy (under the relevant assumptions). With $p^{\circ}(x,y)-$accuracy and $\widetilde{p}(x,y)-$accuracy occurring at the same epoch, we might expect the $p(x,y)-$accuracy to also peak at this epoch as it is an intermediate distribution between $\widetilde{p}$ and $p^{\circ}$. 

\paragraph*{Relaxing Simultaneity} The key condition of Lemma~\ref{ch3:lemma:g_vector_lemma} is that each of the $g_i^{(n)}$ achieve their minima at the same epoch during training. In practice, this is not always exactly satisfied although each of the $g_i$ generally attain their minima within a couple of epochs of one another (See Figure~\ref{ch3:fig:simul}). Lemma~\ref{ch3:lemma:g_vector_lemma} can be generalised to handle this. If we let $T_1<T_2$ be two integers such that each of the
$g_i$ (for $i>1$) achieve their minima at epochs $n\in \{T_1, T_1+1,\ldots, T_2\}$ then both the noisy and clean accuracies are attained in this short interval. Thus, if we stop training when the noisy accuracy is maximised we are fairly close to the point during training when clean accuracy is maximised if $T_2-T_1$ is small. 

\section{Additional Plots and Figures}\label{ch7:sec:additional}

Figure~\ref{ch7:fig:fashion_es_vs_noise} displays the clean test accuracy (blue) and noisy validation accuracy (yellow) after each epoch for a classifier trained on symmetrically noised FashionMNIST ($\eta=0.2$) (top) and \emph{asymmetrically} noised CIFAR10 (also $\eta=0.2$) (2\textsuperscript{nd} from top). On each graph, we indicate the epoch at which the clean test accuracy is maximised with a red cross \stylizedX{} and the maximum noisy validation accuracy with a green plus \stylizedPlus{}. We then mark off the clean test accuracy obtained by the model which maximises the noisy validation accuracy with a green dotted line, and the maximum clean test accuracy attained during training with a red dotted line. For both datasets these lines are almost identical; the clean test accuracy attained by the model which maximises the noisy validation accuracy is within a fraction of a percent of the optimal clean test accuracy. This re-emphasises that Noisy Early Stopping attains a model with nearly optimal clean test accuracy.\\

The bottom of Figure~\ref{ch7:fig:fashion_es_vs_noise} presents the same data as the top two figures however this is expressed differently (FashionMNIST  on the left and CIFAR10 on the right). We plot the clean test accuracy attained at each epoch against the noisy validation accuracy. Each epoch is coloured so that the first few epochs are blue and the final epochs are yellow with the hue shifting gradually from blue to yellow through red. This allows us to visualise how clean and noisy (validation) accuracy evolves during training. Initially, the noisy and clean accuracy both increase, and the data points move into the upper right corner of the graph, before overfitting occurs and both the noisy and clean accuracies decline. Crucially, overfitting on the noisy and clean datasets occurs simultaneously for both datasets meaning the peak clean and noisy accuracies occur at the same time. The FashionMNIST dataset is noised using uniform symmetric label noise. Under these noise conditions, we know that the noisy and clean accuracies are related via a linear map (See Lemma~\ref{ch7:lemma:risk_inequality}) - Figure~\ref{ch7:fig:fashion_es_vs_noise} confirms this linear relationship. Despite the absence of similar guarantees for the asymmetrically-noised CIFAR10 dataset, a similar relationship is visible in its graph. We emphasise that the fact that the noisy and clean accuracies peak simultaneously for CIFAR10 is non-trivial and somewhat unexpected. Similar plots for symmetrically-noised FashionMNIST and Asymmetrically-noised CIFAR10 ($\eta=0.4$) may be found in Figures~\ref{fig:clean_vs_noisy_big_summary_fashion}, \ref{fig:clean_vs_noisy_big_summary_fashion_together} and Figures~\ref{fig:cclean_vs_noisy_big_summary_cifar_together}, \ref{fig:clean_vs_noisy_big_summary_cifar} respectively.

\begin{figure*}[!ht]
\centering
\begin{subfigure}[b]{0.9\textwidth}
    \centering
    \includegraphics[width=0.85\linewidth]{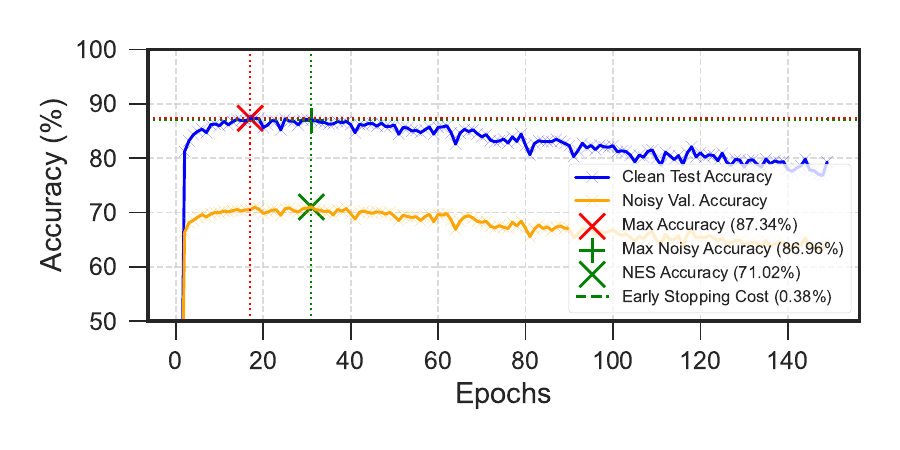}
\end{subfigure}
\vspace{-5mm} 

\begin{subfigure}[b]{0.9\textwidth}
    \centering
    \includegraphics[width=0.85\linewidth]{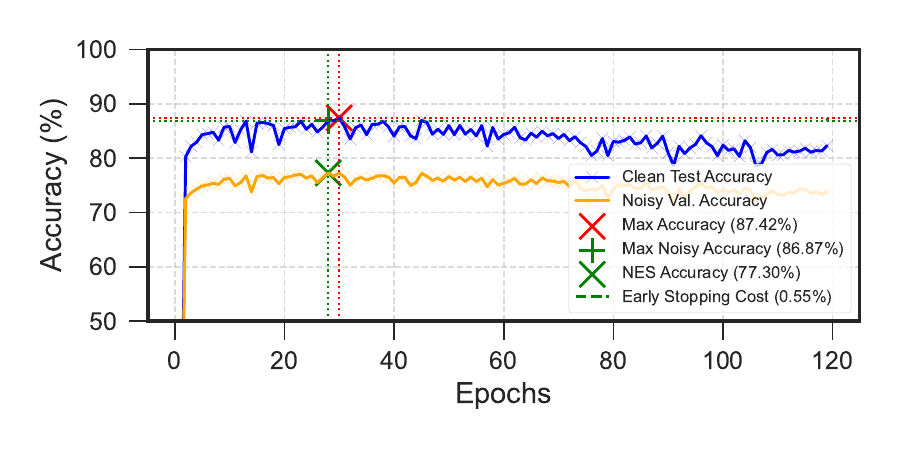}
\end{subfigure}
\vspace{-5mm} 

\begin{subfigure}[b]{0.4\textwidth}
    \centering
    \includegraphics[width=\linewidth]{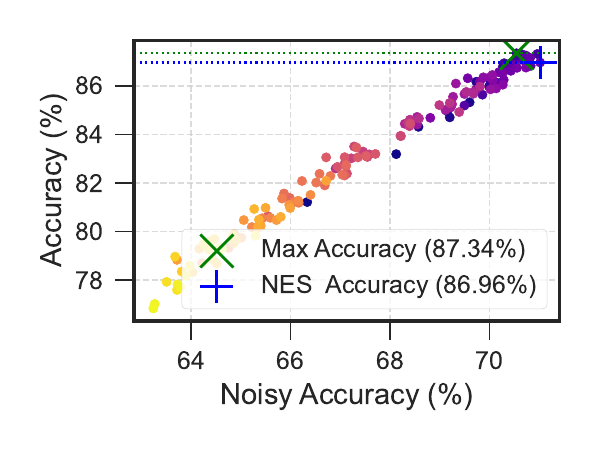}
    \caption{Symmetrically-Noised FashionMNIST}
\end{subfigure}
\begin{subfigure}[b]{0.5\textwidth}
    \centering
    \includegraphics[width=\linewidth]{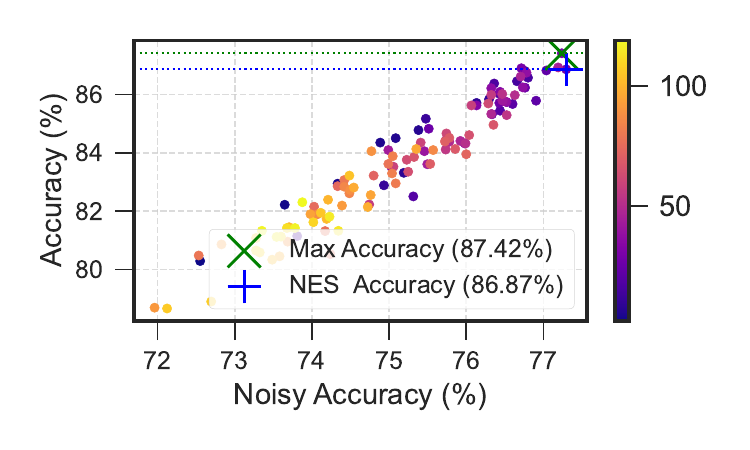}
    \caption{Asymmetrically-Noised CIFAR10}
\end{subfigure}

\caption{The top figure shows the clean test accuracy (blue) and noisy validation accuracy (yellow) during training on the symmetrically-noised Fashion dataset ($\eta=0.2$). We highlight the maximum clean accuracy achieved (red cross) and compare this with the accuracy achieved by Early Stopping using the noisy validation set (green plus) - the difference in clean test accuracy between these is a mere $0.38\%$. The second figure from the top displays repeats for the asymmetrically-noised CIFAR10 data. Once again the difference is a mere $0.55\%$.\\
The bottom two figures provide the same information as the top two but are expressed differently (FashionMNIST on the left and MNIST on the right). We plot the clean test accuracy against the noisy validation accuracy during training. Each epoch is coloured so that the first few epochs are blue and the final epochs are yellow, with the hue shifting gradually from blue to yellow through red. Initially, both the noisy and clean accuracy increase, moving into the upper right corner of the graph before overfitting occurs, leading to a decline in both accuracies. For both datasets, the noisy and clean accuracies are maximised approximately simultaneously.}
\label{ch7:fig:fashion_es_vs_noise}
\end{figure*}

\begin{figure}[!ht]
\centering
  \centering
  \includegraphics[width=\linewidth]{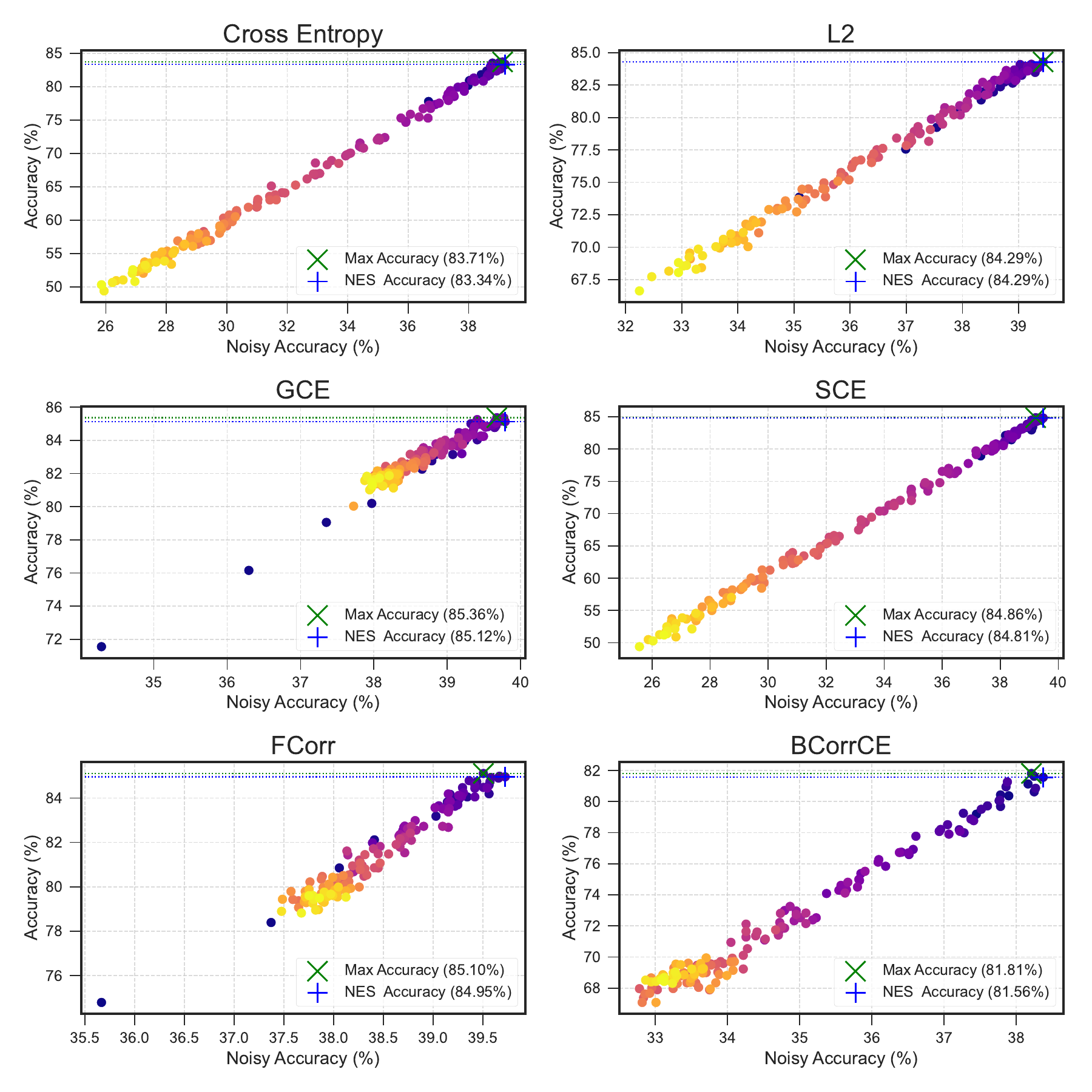}
  \caption{We plot the clean test accuracy against the noisy validation accuracy during training on symmetrically-noised Fashion dataset ($\eta=0.4$) for the CE, MSE, GCE, SCE, FCE and BCE loss functions. Each epoch is coloured so that the first few epochs are blue and the final epochs are yellow with the hue shifting gradually from blue to yellow through red. Initially, the noisy and clean accuracy both increase, and the data points move into the upper right corner of the graph, before overfitting occurs and both the noisy and clean accuracies decline. Both both datasets the noisy and clean accuracies are maximised approximately simultaneously. }
  \label{fig:clean_vs_noisy_big_summary_fashion}
\end{figure}

\begin{figure}[ht!]
\centering
  \centering
  \includegraphics[width=\linewidth]{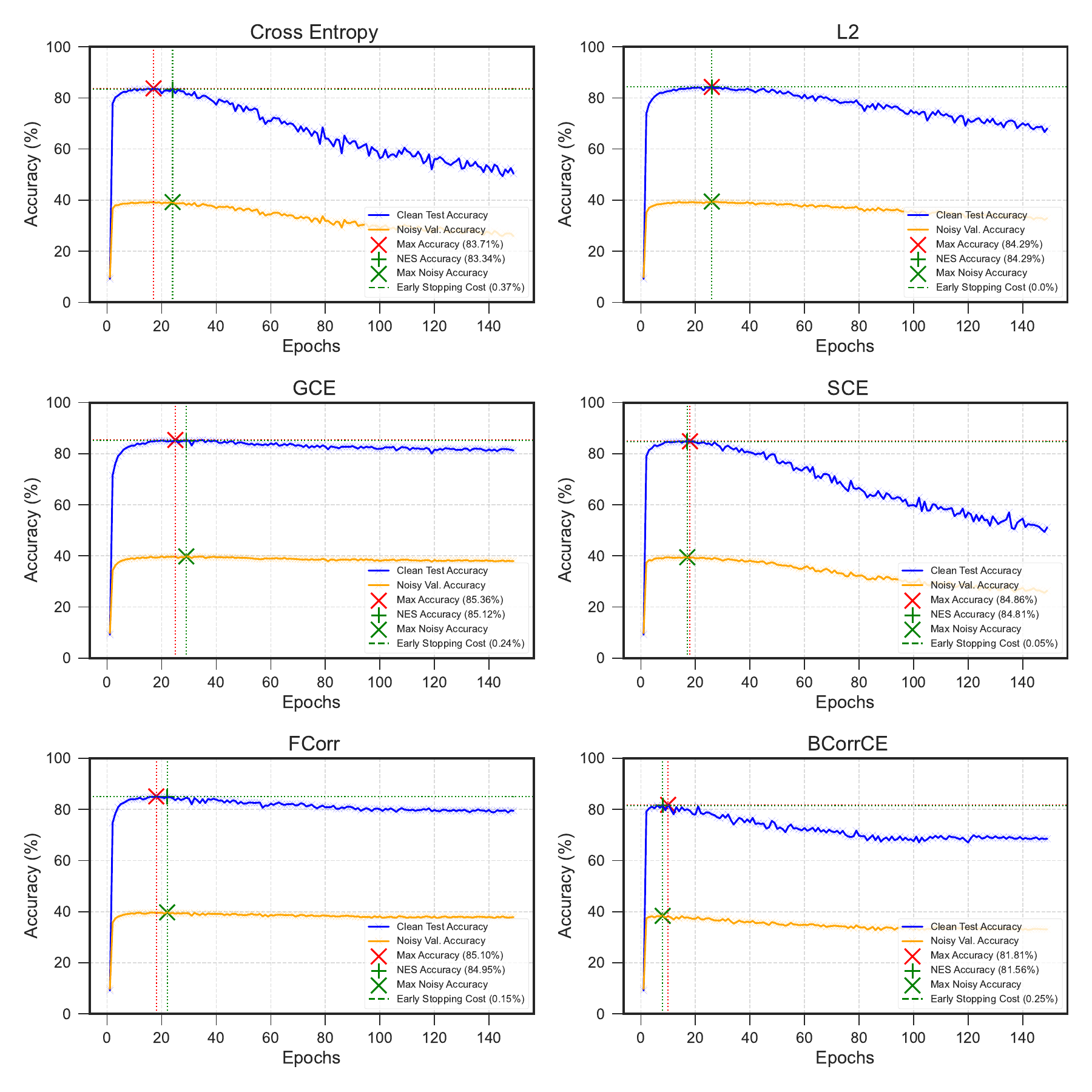}
  \caption{We plot the clean test accuracy and the noisy validation accuracy during training on symmetrically-noised Fashion dataset ($\eta=0.6$) for the CE, MSE, GCE, SCE, FCE and BCE loss functions. We highlight the maximum clean accuracy achieved (red) and compare this with the accuracy achieved by Early Stopping using the noisy validation set - the difference is generally within 1$\%$. }
  \label{fig:clean_vs_noisy_big_summary_fashion_together}
\end{figure}

\begin{figure}[ht!]
\centering
  \centering
  \includegraphics[width=\linewidth]{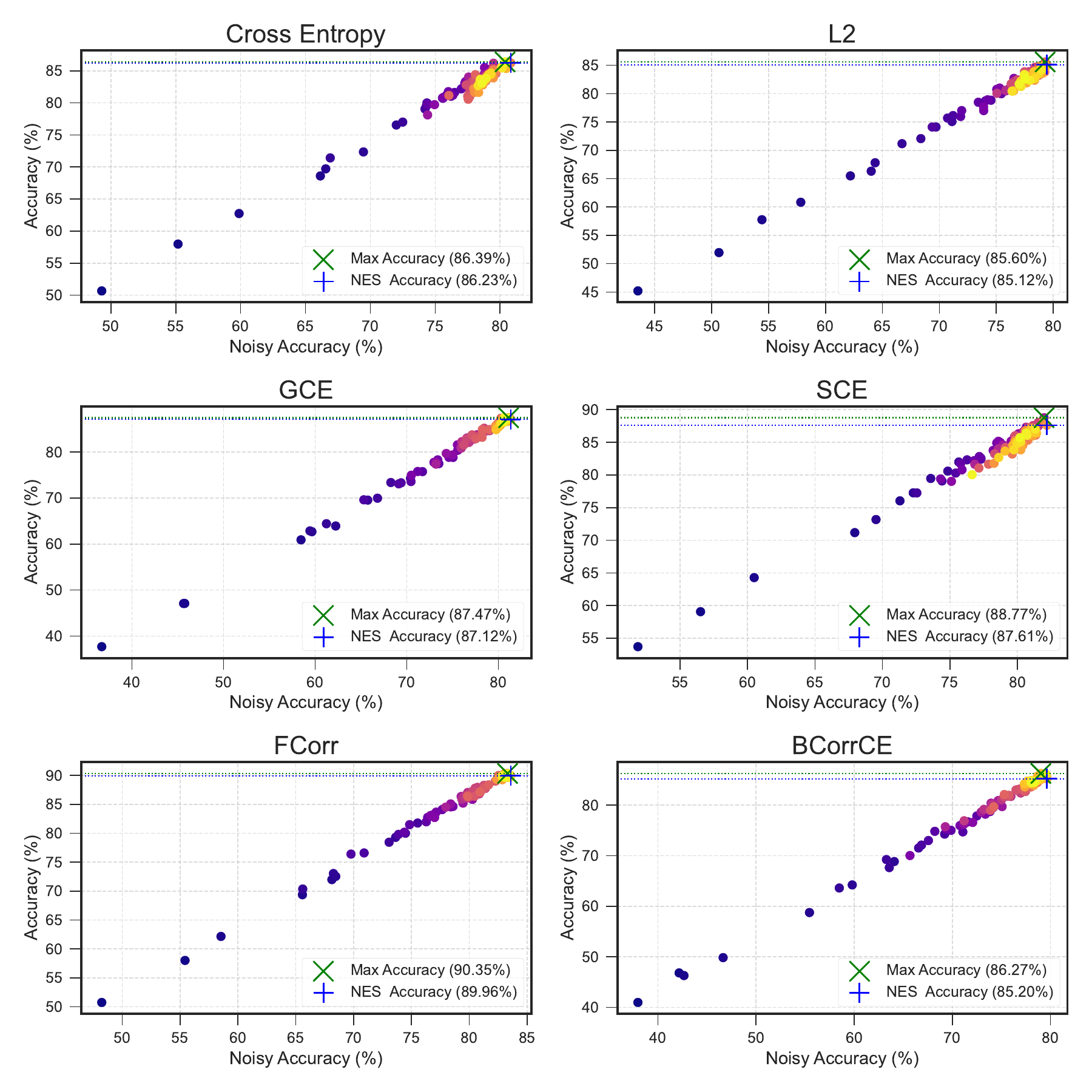}
  \caption{We plot the clean test accuracy against the noisy validation accuracy during training on asymmetrically-noised CIFAR10 dataset ($\eta=0.4$) for the CE, MSE, GCE, SCE, FCE and BCE loss functions. Each epoch is coloured so that the first few epochs are blue and the final epochs are yellow with the hue shifting gradually from blue to yellow through red. Initially, the noisy and clean accuracy both increase, and the data points move into the upper right corner of the graph, before overfitting occurs and both the noisy and clean accuracies decline. Both both datasets the noisy and clean accuracies are maximised approximately simultaneously.}
  \label{fig:clean_vs_noisy_big_summary_cifar}
\end{figure}

\begin{figure}[ht!]
\centering
  \centering
  \includegraphics[width=\linewidth]{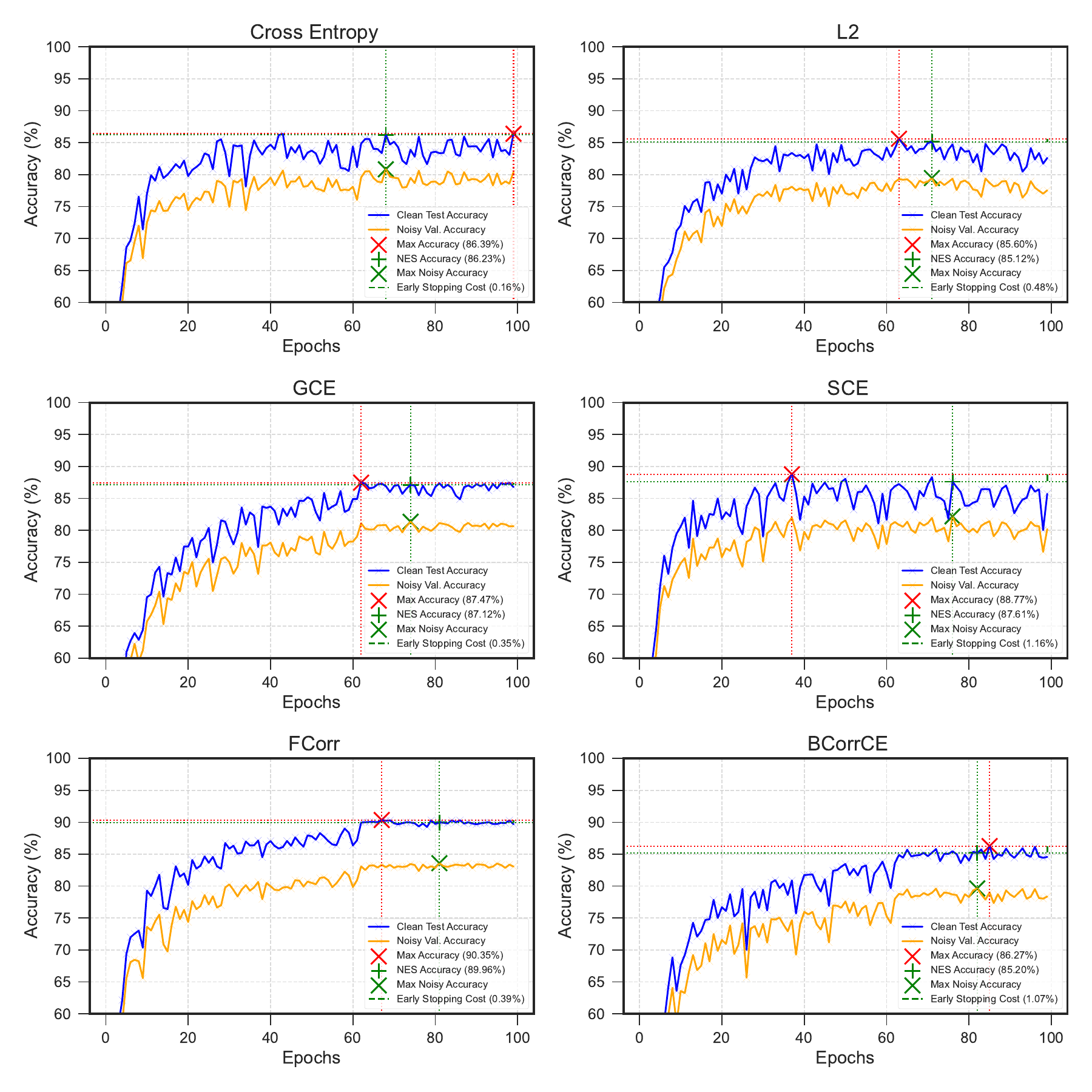}
  \caption{We plot the clean test accuracy and the noisy validation accuracy during training on asymmetrically-noised CIFAR10 dataset ($\eta=0.4$) for the CE, MSE, GCE, SCE, FCE and BCE loss functions. We highlight the maximum clean accuracy achieved (red) and compare this with the accuracy achieved by Early Stopping using the noisy validation set - the difference is a generally within 1$\%$.}
  \label{fig:cclean_vs_noisy_big_summary_cifar_together}
\end{figure}

\section{Algorithm Details and Code}
\subsection{Creation of the Noised Datasets}\label{app:dataset_models}

\subsubsection{Symmetric Noise in Image Datasets}

In symmetric noise models, we alter a proportion \( \eta \) of labels in datasets such as CIFAR-10, CIFAR-100, MNIST, and FashionMNIST. For each affected label, a new label is chosen uniformly at random. In CIFAR-10 and CIFAR-100, the new label is selected from all possible labels excluding the original. In MNIST and FashionMNIST, the original label remains a possible choice. The datasets are split into 70\% training, 15\% noisy validation, and 15\% clean test sets to evaluate the effects of label noise. The noise rates used in the experiments shown in Table~\ref{tab:performance-metrics} are MNIST ($\eta=0.8$), FashionMNIST ($\eta=0.6)$, NoisedCIFAR10 ($\eta=0.6)$ and NoisedCIFAR100 ($\eta=0.6$). 

\subsubsection{Non-Uniform Noise in EMNIST}

In the \textit{Non-Uniform EMNIST} setup, we introduce label noise based on a linear classifier's predictions. After training the classifier on EMNIST, we modify the label of each training point with probability \( \eta \) to the classifier's output. This method creates \( x \)-dependent noise, where the label alterations depend on the classifier's performance across different regions of the data space. The classifier attains an accuracy of 28\%, meaning the new label is different from the original label 72\% of the time.

\subsubsection{Asymmetric Label Noise}
\paragraph*{CIFAR10} The asymmetric label noise for CIFAR10 is consistent with that used in \citep{fprop} and most other noise robust literature. The label noise is constructed by mapping pairs of classes into each other, specifically, with probability $\eta$ one transitions $Truck\to Autombile$, $Bird\to Plane$, $Deer \to Horse$, $Cat \leftrightarrow Dog$ \citep{fprop}. This label noise is used in all papers which use the CIFAR10 dataset with asymmetric label noise except \citep{anchor}. In \citep{anchor} they use circular label noise flipping each class to the following class with probability $\eta$. Our experiments in Table~\ref{tab:performance-metrics} uses $\eta=0.4$.

\paragraph*{FashionMNIST}
The asymmetric label noise for FashionMNIST follows the methodology used in \citep{GCE_Loss}. This synthetic label noise is constructed similarly to CIFAR10 and MNIST, via pairwise transitions; $Boot\to Sneaker$, $Sneaker\to Boot$, $Sneaker\to Sandals$, $Pullover\to Shirt$, $Coat\leftrightarrow Dress$. Our experiments in Table~\ref{tab:performance-metrics} uses $\eta=0.4$.

\paragraph*{CIFAR100} The asymmetric label noise for CIFAR100 follows the methodology used in \citet{fprop}. In \citet{fprop} the authors introduce asymmetric noise to the CIFAR100 dataset by implementing circular noise within each of the 20 `superclasses.' The CIFAR100 dataset is organised into groups of 5 classes, termed as superclasses. For instance, the `Aquatic Mammals' superclass comprises the classes Beaver, Dolphin, Otter, Seal, and Whale \citep{fprop}. Within each superclass, labels are cyclically permuted (e.g., Beaver $\mapsto$ Otter $\mapsto$ Dolphin, and so forth) with a probability of $\eta$. This label noise is also used in the following papers \cite{GCE_Loss}, \cite{normalised_losses}, \cite{taylorce}, \cite{sce}, \cite{bootstrap}, \cite{fprop}, \cite{gjs}, \cite{curiculum}, \cite{elr}. Our experiments in Table~\ref{tab:performance-metrics} uses $\eta=0.6$.


\paragraph*{MNIST} The asymmetric label noise for MNIST deviates slightly from the literature. We divide the labels into 4 groups: \{0,1,2\}, \{3,4,5\}, \{6,7,8\}, \{9\}. Within each of the sets of cardinality three label noise as described by the following transition matrix where $T_{ij}\coloneq p(\widetilde{y}=i \mid y=j)$
\begin{align*}
    T = \begin{bmatrix}
    1-\eta & \eta & \eta\\
    \eta & 1- \eta & \eta\\
    0 & 0 & 1-2\eta
\end{bmatrix}
\end{align*}
Our experiments in Table~\ref{tab:performance-metrics} uses $\eta=0.2$.

\subsubsection{NonUniformMNIST Dataset}\label{app:nonuniformMNIST}

The NonUniformMNIST dataset is a modified version of the standard MNIST dataset. It introduces non-uniform noise into the labels based on a noise level parameter, \( \eta \). Two distinct class-conditional noise models are applied across the dataset, with each model influencing separate segments of data. This segmentation is governed by the principal component analysis (PCA) of each class, ensuring that exactly half of the samples from each class are noised according to the first model, and the remaining half according to the second model. 

\paragraph*{Label Noise Process:}
Each image label in the dataset is subject to modification according to the following probabilistic rule:
\begin{itemize}
    \item With probability \( 1-\eta \), the label remains unchanged.
    \item With probability \( \eta \), the label is altered using one of two predefined stochastic transition matrices, denoted as \texttt{matrix1} and \texttt{matrix2}.
\end{itemize}

\paragraph*{Transition Matrices:}
\begin{itemize}
    \item \textbf{\texttt{matrix1} (Symmetric Noise):} This matrix introduces uniform symmetric noise across all labels, constructed such that each incorrect label assignment occurs with equal probability.
    \item \textbf{\texttt{matrix2} (Symmetric Noise with Subsets):} This matrix applies noise only within predefined subsets of labels: (0,1,2), (3,4,5), and (6,7,8), with label 9 unchanged. Noise within these subsets is also symmetric.
\end{itemize}

\paragraph*{Selection of Transition Matrix:}
For each class in the dataset:
\begin{enumerate}
    \item \textbf{Statistical Computation:}
    \begin{itemize}
        \item Calculate the mean of all samples within the class.
        \item Determine the first principal component (PC1) of the class samples.
    \end{itemize}
    \item \textbf{Label Adjustment Criterion:}
    \begin{itemize}
        \item Center each sample by subtracting the class mean.
        \item Project the centred sample onto PC1.
        \item The sign of the resulting scalar product determines which transition matrix is used:
        \begin{itemize}
            \item If negative, \texttt{matrix1} is used to modify the label.
            \item If positive, \texttt{matrix2} is employed.
        \end{itemize}
    \end{itemize}
\end{enumerate}

Our experiments in Table~\ref{tab:performance-metrics} uses $\eta=0.3$.

\subsubsection{Synthetic Noisy Dataset}
\paragraph{Dataset Creation}
The synthetic dataset used in the experiment shown in Figure~\ref{fig:both_classifiers} was generated using the \texttt{make\_classification} function from \texttt{scikit-learn}, adept at creating complex multiclass classification problems. The dataset consists of 2,000 samples, each with 20 features, where ten are informative and directly influence the class labels, and the remaining ten are designed to mimic irrelevant data aspects found in real-world scenarios. The data is divided into three classes, each centered around a single cluster to simplify the classification task and emphasise the impact of label noise over inter-class separation. Feature scaling was performed using \texttt{StandardScaler} from \texttt{scikit-learn} to standardise the dataset.

\paragraph{Noise Introduction}
Label noise was introduced at a rate of 41\% using a pairwise transition method, where each label is cyclically shifted to the next. The introduction process was controlled using a fixed random seed (\texttt{random\_state=42}), ensuring consistent and reproducible noisy labels across different experiments.

\paragraph{Model Description}
The neural network used in the experiment comprised two hidden layers with 64 neurons each, and a ReLU activation function followed by an output layer for three-class classification. The model is trained using Stochastic Gradient Descent (SGD) with a learning rate of 0.1.

\section{Why Does NES Work?}\label{ch7:sec:why_nes_works}
In Section~\ref{ch7:sec:theoretical} we showed that unless label noise is uniform and symmetric, the minimiser of the noisy risk within a set of estimators $\mathcal{Q}$ may not be a minimiser of the clean risk. Moreover, the bounds which we derived (Section~\ref{ch7:thm:worst_case_bounds}) give weak guarantees, by this we mean that the minimiser of the noisy risk can generalise very poorly to the clean distribution (See Discussion in Section~\ref{ch7:sec:discussion}). Therefore, based on Facts 1-5 alone we would not expect NES to be as effective as it turns out to be in Section~\ref{ch7:sec:experiments}. The efficacy of NES therefore remains largely unexplained. In this section we study this topic in more detail and provide a partial explanation for this non-trivial phenomenon, suggesting that NES is effective due to the way neural network classifiers fit: Roughly speaking, if the off-diagonal elements of the confusion matrix for the classifier being learned are minimised around the same training epoch, then NES will be effective.

\paragraph*{Assumptions:} Throughout this section we assume that we train a neural network classifier on a dataset corrupted by label noise. We make a minor notational alteration, letting $\bm{q}^{(n)}$ denote the model attained after training for $n$ epochs and we let $\mathcal{Q}\coloneq \{\bm{q}^{(n)}\}_{n=1}^N$.\footnote{We used $\bm{q}_{n}$ previously to denote the model attained after training for $n$ epochs, however, this notation results in subscript overcrowding if utilised in the following section; hence the change.} We use the notation $\bm{q}_*$ to denote the minimiser of the clean risk in $\mathcal{Q}$ and $\bm{q}^{\eta}_*$ to denote the minimiser of the noisy risk. That is
\begin{align*}
    \bm{q}^{\eta}_* &\coloneq \argmin_{\bm{q}\in \mathcal{Q}} R^{\eta}(\bm{q})\\
    \bm{q}_* &\coloneq \argmin_{\bm{q}\in \mathcal{Q}} R(\bm{q})
\end{align*}

We continue to assume the label noise being discussed satisfies the class-preserving assumption (Definition~\ref{ch3:def:conserveNoise}).
We assume that the data-label distribution is separable and that the label noise is class-conditional asymmetric where each column is a permutation of every other column. This latter assumption has the impact of simplifying the mathematics and so helps with the exposition. Results extend to general asymmetric label noise.

\subsection{Section Outline:} In this section we construct an example setting in which a Noisy Early-Stopping policy would fail to select the clean risk minimiser within $\mathcal{Q}$. Constructing this setting allows us to develop an intuition for why examples like this do not occur in practice. The key idea of this section involves measuring the proportions by which a classifier predicts the most likely noisy label, how often it predicts the second most likely noisy label, how often it predicts the third most likely noisy label etc. We store these proportions, which we denote $g_1,g_2,g_3,\ldots, g_c$, in a vector which we call the `$g-$vector' for the classifier, and record how the components of this vector change during training. For example, if a classifier has a $g-$vector $\bm{g}=(g_1,g_2,g_3)=(0.6,0.3,0.1)$ this means it predicts the most likely noisy label $60\%$ of the time, the second most likely noisy label $30\%$ of the time and the least likely noisy label $10\%$ of the time. We show theoretically that if during training, for $i\geq 2$ each of the $g_i$ attain their minima simultaneously then $\bm{q}^{\eta}_* = \bm{q}_*$; the noisy risk minimiser is also the minimiser of the clean risk. We conclude by demonstrating experimentally that when training a neural network classifier this condition is satisfied thus explaining why NES is effective in practice.

\subsection{\texorpdfstring{$g$-vector}{g-vector}}\label{ch7:sec:g_vector_summary}

In this section, we formally define the $g-$vector associated with a classifier. We assume, for purposes of convenience, for all $i \neq j$, $\widetilde{p}(\widetilde{y}=i\mid x)\neq\widetilde{p}(\widetilde{y}=j\mid x)$, meaning that $\argkmax_i \widetilde{p}(\widetilde{y}=i\mid x)$ consists of a single element.\footnote{Where $\argkmax$ is a generalisation of $\argmax$ to the $k$\textsuperscript{th} largest elements in a set. } This choice aids in the subsequent exposition which still holds in the general case. 

\paragraph*{$g-$functions} Given an estimator $\bm{q}$ with plug-in classifier $f$, noisy data-label distribution $\widetilde{p}(x,y)$ and some $x\in\mathcal{X}$, we define the following set of $c$ binary-valued functions $g_1, g_2, \ldots, g_c$ where $g_k(x)$ is defined;
\begin{align}
        g_k(x) \coloneq \begin{cases} 
      1, & \text{if } f(x) = \argkmax_i \widetilde{p}(\widetilde{y}=i \mid x) \\
      0, & \text{otherwise }
   \end{cases}
\end{align}

Note that for any $x\in \mathcal{X}$ \textit{exactly one} of the functions $g_k(x)$ is equal to one and the rest equal zero. We define the following vector-valued function by concatenating the $g_k$; 
\begin{align}
    \bm{g}(x) \coloneq (g_1(x), g_2(x), \ldots, g_c(x)).
\end{align}
If $g_1(x)=1$ this means that the associated classifier is predicting the most likely (noisy) label to appear at $x$. If $g_2(x)=1$ means that the classifier is predicting the second most likely to appear given $x$ and so on for the given (noisy) data distribution.

The expectation of this vector over the data distribution is denoted 
\begin{align*}
    \bm{g} \coloneq \int p(x) \bm{g}(x) dx \in \Delta.
\end{align*}
We call this the \textbf{g-vector} obtained from the classifier $f$. The vector, $\bm{g}$, gives us an average of how many times our classifier predicts the most likely label, second mostly likely and so on. For example, if the classifier is Bayes-optimal for the noisy distribution then $\bm{g} = (1,0,\ldots,0)$ indicating that at every $x$ in the support of $p(x)$ the classifier $f$ predict the most likely noisy label.

\subsection{When Would NES Fail?}\label{ch3:sec:arch_counter}
When is it true that the minimiser of the noisy risk in $\mathcal{Q}$ is not a minimiser of the clean risk? The key insight is identifying that two classifiers $f_1, f_2$ can have the same clean accuracy and dramatically different noisy accuracies if, whenever $f_1$ predicts incorrectly, it predicts a probable noisy class whereas, whenever $f_2$ predicts incorrectly it predicts an \emph{improbable} noisy class. Put simply; when $f_2$ is wrong it is \emph{very} wrong, predicting a class which occurs with very low probability.

To make this more concrete consider the following simplified example where we have a (three class) data-label distribution $p(x,y)$ where, for every $x\in \supp(p(x))$, the clean and noisy conditional class distribution are as follows
\begin{align*}
    \bm{p}(y\mid x) &= (1,0,0)\\
    \text{and}\quad \widetilde{\bm{p}}(\widetilde{y}\mid x) &= (0.6,0.35,0.05).
\end{align*}
This is to say that the clean label at every $x$ is $y=1$ with probability $100\%$. Likewise, for every $x$, the probability of the noisy label being $\widetilde{y}=1$ is $60\%$, the probability of the noisy label being $\widetilde{y}=2$ is $35\%$ and the probability of the noisy label being $\widetilde{y}=3$ is $5\%$. Suppose that we have two classifiers $f_1, f_2$ which both correctly predict the most likely noisy label ($\widetilde{y}=1$) $60\%$ of the time. However, the remainder of the time $f_1(x)=2$ and $f_2(x)=3$. 
Utilising the language of Section~\ref{ch7:sec:g_vector_summary} we say the $g-$vector for $f_1$ is $\bm{g}_1=(0.6,0.4,0)$ and is $\bm{g}_2=(0.6,0,0.4)$ for $f_2$. In this case, even though both models have the same \emph{clean} accuracy of $60\%$, $f_1$ obtains a noisy accuracy of $50\%$ whereas $f_2$ obtains a noisy accuracy of $38\%$. (The noisy accuracy is obtained by computing a dot product between $\widetilde{\bm{p}}(\widetilde{y}\mid x)$ and the $g-$vectors.)\\

We can employ the principle outlined above to construct instances in which $\bm{q}_*^{\eta}$ has a lower clean accuracy than $\bm{q}_*$. To construct such an instance one must ensure that wherever $\bm{q}_*$ does not predict the correct clean label, it predicts an unlikely noisy label. Whereas, wherever $\bm{q}^{\eta}_*$ does not predict the correct clean label, it predicts a likely noisy label. This is represented visually in Figure~\ref{fig:NES_counterexample}. Figure~\ref{fig:NES_counterexample} plots bar charts of the $g-$vectors of $\bm{q}_*^{\eta}$ and $\bm{q}_*$ in a ternary label setting in which NES would fail. The figure shows the proportion by which each classifier predicts the true label\footnote{Crucially, since we are assuming the label noise is class-preserving the true clean label is also the most likely noisy label.} (blue), the second most likely noisy label (orange) and the least likely noisy label (green) respectively. While the noisy accuracy is higher for $\bm{q}_*^{\eta}$ it has a lower clean accuracy than $\bm{q}_*$. 

\begin{figure*}[h!]
  \centering
  \includegraphics[width=0.95\textwidth]{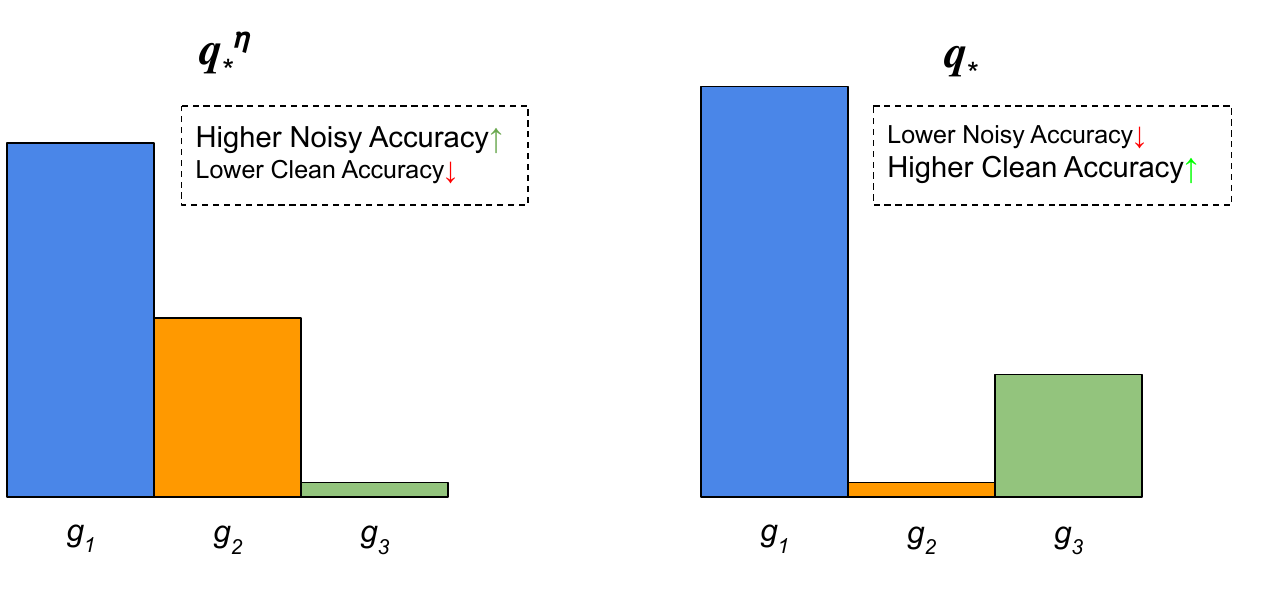} 
  \caption{Example where NES would fail: Histograms of predictions for $\bm{q}_*^{\eta}$ and $\bm{q}_*$ in the case of three classes. This shows the frequency with which each classifier predicts the true clean label (blue), the second most likely noisy label (orange) and the least likely noisy label (green) respectively. While the noisy accuracy is higher for $\bm{q}_*^{\eta}$ it has a lower clean accuracy than $\bm{q}_*$.
} 
  \label{fig:NES_counterexample}
\end{figure*}

\subsection{Overfitting}
Since NES works in practice, pathological examples of the type described in Section~\ref{ch3:sec:arch_counter} and represented in Figure~\ref{fig:NES_counterexample} must not arise when we train a classifier. To understand why this suppose that the minimiser of the noisy risk $\bm{q}^{\eta}_*$ is \emph{not} the minimiser of the clean risk ($\bm{q}^{\eta}_*\neq \bm{q}_*$). In this case, there are two possibilities:
Either the minimiser or the noisy risk $\bm{q}^{\eta}_*$ occurs earlier in training that he clean risk minimiser $\bm{q}_*$, or $\bm{q}^{\eta}_*$ occurs later in training than $\bm{q}_*$ (depicted in Figure~\ref{fig:NES_counterexample_order}).  In either case, there must exist distinct $i,j\geq 2$ for which $g_i$ increases at the same time that $g_j$ decreases. If we can demonstrate that in practice this does not occur then we will have an explanation for why NES succeeds. We formalise this with the following Lemma. 

\begin{figure*}[h!]
  \centering
  \includegraphics[width=0.99\textwidth]{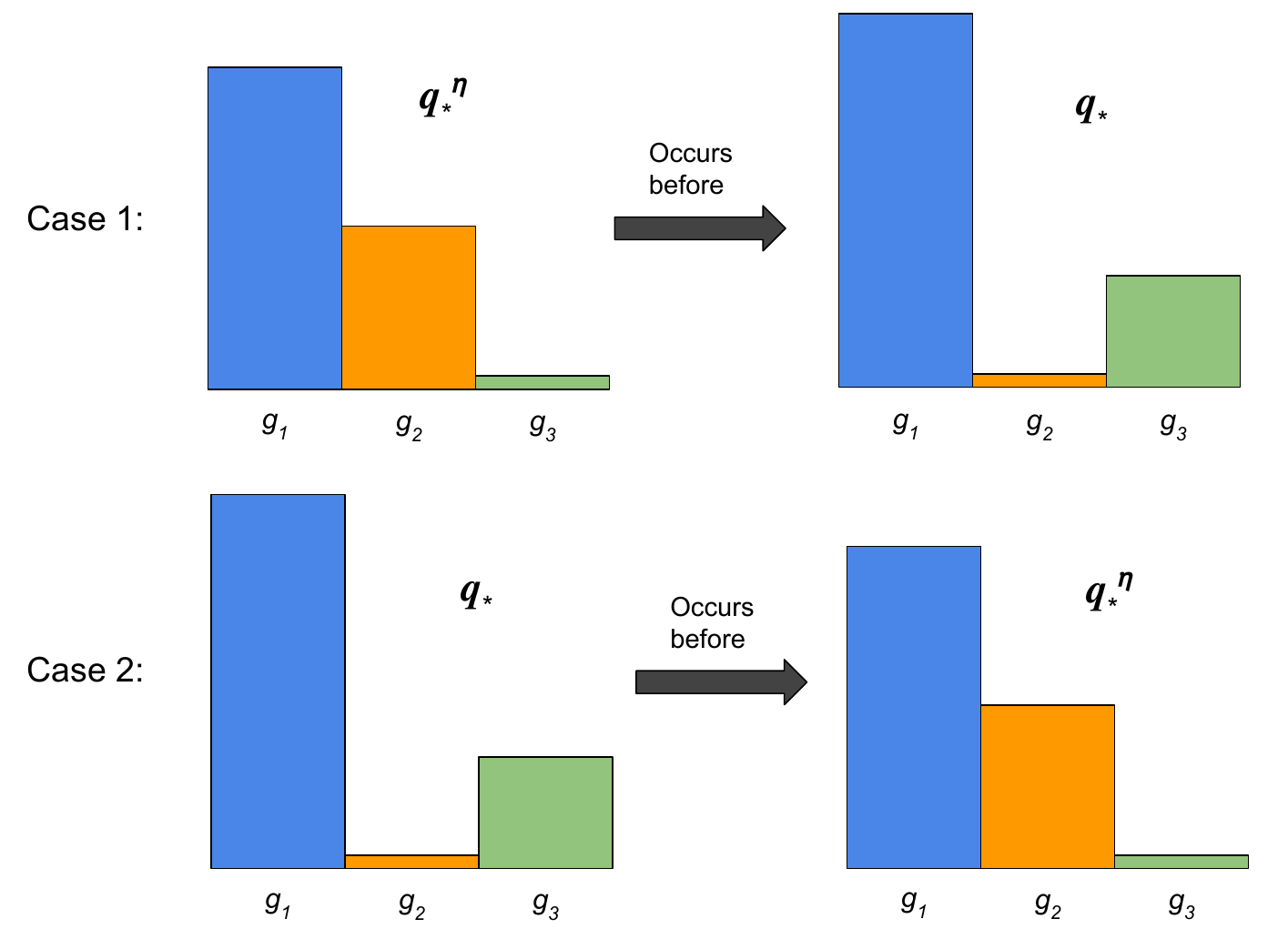} 
  \caption{The two possibilities assuming that the minimiser of the noisy risk does not minimise the clean risk.  Case 1: the minimiser of the noisy risk occurs earlier in training (top row). Case 2: the minimiser of the noisy risk occurs later in training (bottom row).} 
  \label{fig:NES_counterexample_order}
\end{figure*}

\begin{lemma}\label{ch3:lemma:g_vector_lemma}
    Suppose we train a classifier on a dataset corrupted by (class-preserving) label noise for $N$ epochs. Let $f^{(n)}$ denote the classifier obtained after training for $n$ epochs. Let $\bm{g}^{(n)} = (g^{(n)}_1, g^{(n)}_2, \ldots, g^{(n)}_c)$ denote the $g-$vector of the $n$\textsuperscript{th} classifier $f^{(n)}$. Suppose there exists an integer $T< N$ so that for all $i\geq 2$, $g^{(n)}_i$ is decreasing for $n\in\{1,2,\ldots, T\}$ and increasing for $n\in\{T,T+1\ldots, N\}$ then the minimiser of the noisy risk within $\{f^{(n)}\}_{n=1}^N$ also minimises the clean risk.
\end{lemma}
Proof given in Section~\ref{ch3:sec:generalising_g_vector_lemma}.

\subsection{Experimental Confirmation}
The core hypothesis utilised in Lemma~\ref{ch3:lemma:g_vector_lemma} is that, during training for $i\geq2$, the $g_i$ start by decreasing. They each reach their minima approximately simultaneously at some epoch $T$ before increasing for the remainder of training. We have shown that under these conditions the minimiser of the noisy risk in $\mathcal{Q}$ will minimise the clean risk, meaning that NES would be effective. It remains to demonstrate that this assumption holds in practice.

\paragraph*{Experiment Details:} Our experiment uses a noised, ternary version of the MNIST dataset. We remove all classes other than $\{0,1,2\}$, we then apply synthetic asymmetric label noise to the training set according to the following transition matrix (setting $\eta = 0.3$)
\begin{align*}
    \begin{bmatrix}
1-1.5\eta & 0.5\eta & \eta \\
\eta & 1-1.5\eta & 0.5\eta \\
0.5\eta & \eta & 1-1.5\eta.
\end{bmatrix}
\end{align*}
We train a neural network classifier on this noisy dataset for $100$ epochs. After each epoch, we look at the predictions of the model on the held-out test set. At each datapoint in the test set, we record whether the model predict the most likely noisy label, second most likely etc. We can do this since we have access to the underlying noise model. We record the proportions and display how this evolves during training. This is shown in Figure~\ref{ch3:fig:stacked_bars}.  

\begin{figure*}[ht!]
  \centering
  \includegraphics[width=1\textwidth]{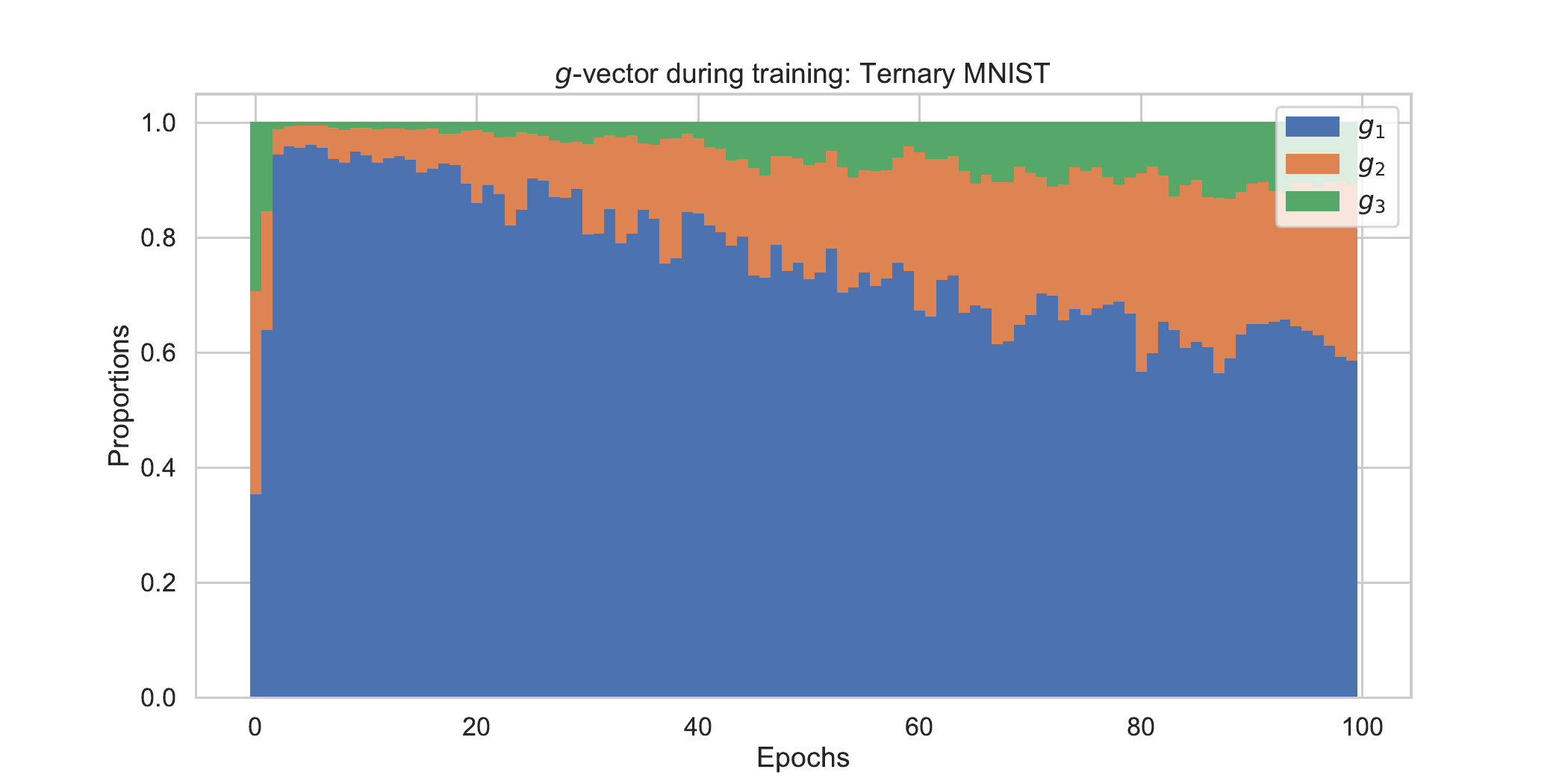} 
  \caption{We train a neural network classifier on an asymmetrically noised ternary MNIST dataset. At each epoch we evaluate on a held-out dataset of samples. Since we have access to the noisy posteriors we can compute the $g-$vectors for the classifier at each epoch. We plot these as a stacked bar chart. 
}
  \label{ch3:fig:stacked_bars}
\end{figure*}

\begin{figure*}[h!]
  \centering
  \includegraphics[width=1\textwidth]{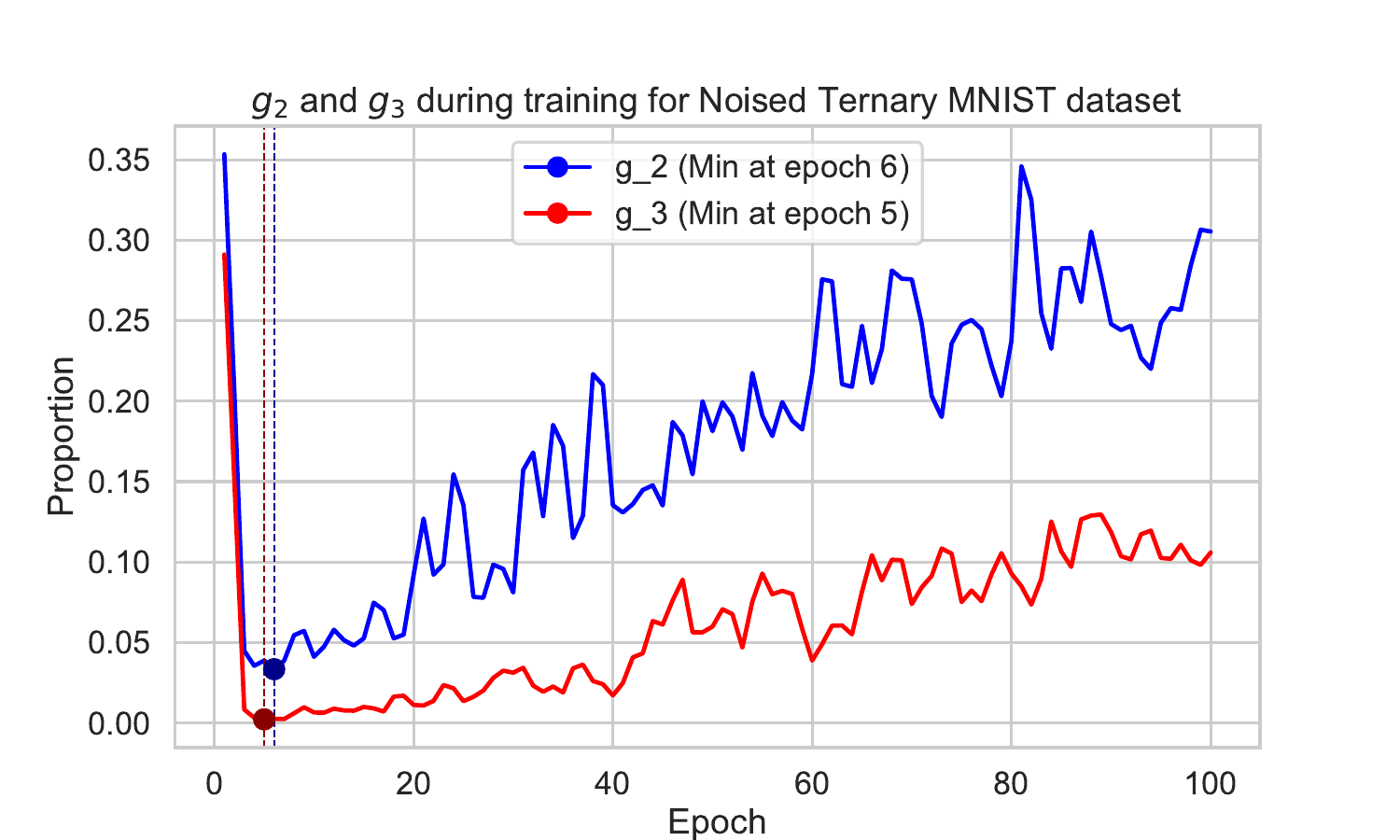} 
  \caption{We train a neural network classifier on an asymmetrically noised ternary MNIST dataset. At each epoch we evaluate on a held-out dataset of samples. Since we have access to the noisy posteriors we can compute the $g-$vectors for the classifier at each epoch. We plot the $g_2$ (blue) and $g_3$ (red) during training and note when each attains its minima. Each minima occurs almost simultaneously (one epoch difference). 
}
  \label{ch3:fig:simul}
\end{figure*}

Figure~\ref{ch3:fig:stacked_bars} shows that for the first 6 epochs or so the proportion by which our model predicts the second or third most likely noisy class decreases. At around epoch $6$ the proportion by which our model predicts the second or third most likely noisy class begins to increase. Crucially both attain their minima almost simultaneously, satisfying the condition in Lemma~\ref{ch3:lemma:g_vector_lemma}. This can be seen in more detail in Figure~\ref{ch3:fig:simul} which shows $g_2$ and $g_3$ only during training. Both attain their minima almost simultaneously as epochs $5,6$ respectively. 

\paragraph*{CIFAR} The experiment above is repeated for a five-class asymmetrically noised version of the CIFAR dataset. The transition matrix used to construct the label noise is given in Equation~\ref{eqn:transition-cifar5}. A plot of the $g-$vectors during training and scatter plots of $g_1, g_2, g_3, g_4$ can be seen in Figure~\ref{ch3:fig:simul5} and Figure~\ref{ch3:fig:stacked_bars5} respectively. Once again the $g_i$ attain their minima around the same time. 
\begin{align}\label{eqn:transition-cifar5}
    T \coloneq \begin{bmatrix}
0.5 & 0.05 & 0.1 & 0.15 & 0.2 \\
0.2 & 0.5 & 0.05 & 0.1 & 0.15 \\
0.15 & 0.2 & 0.5 & 0.05 & 0.1 \\
0.1 & 0.15 & 0.2 & 0.5 & 0.05 \\
0.05 & 0.1 & 0.15 & 0.2 & 0.5
\end{bmatrix}
\end{align}

\begin{figure*}[h!]
  \centering
  \includegraphics[width=0.9\textwidth]{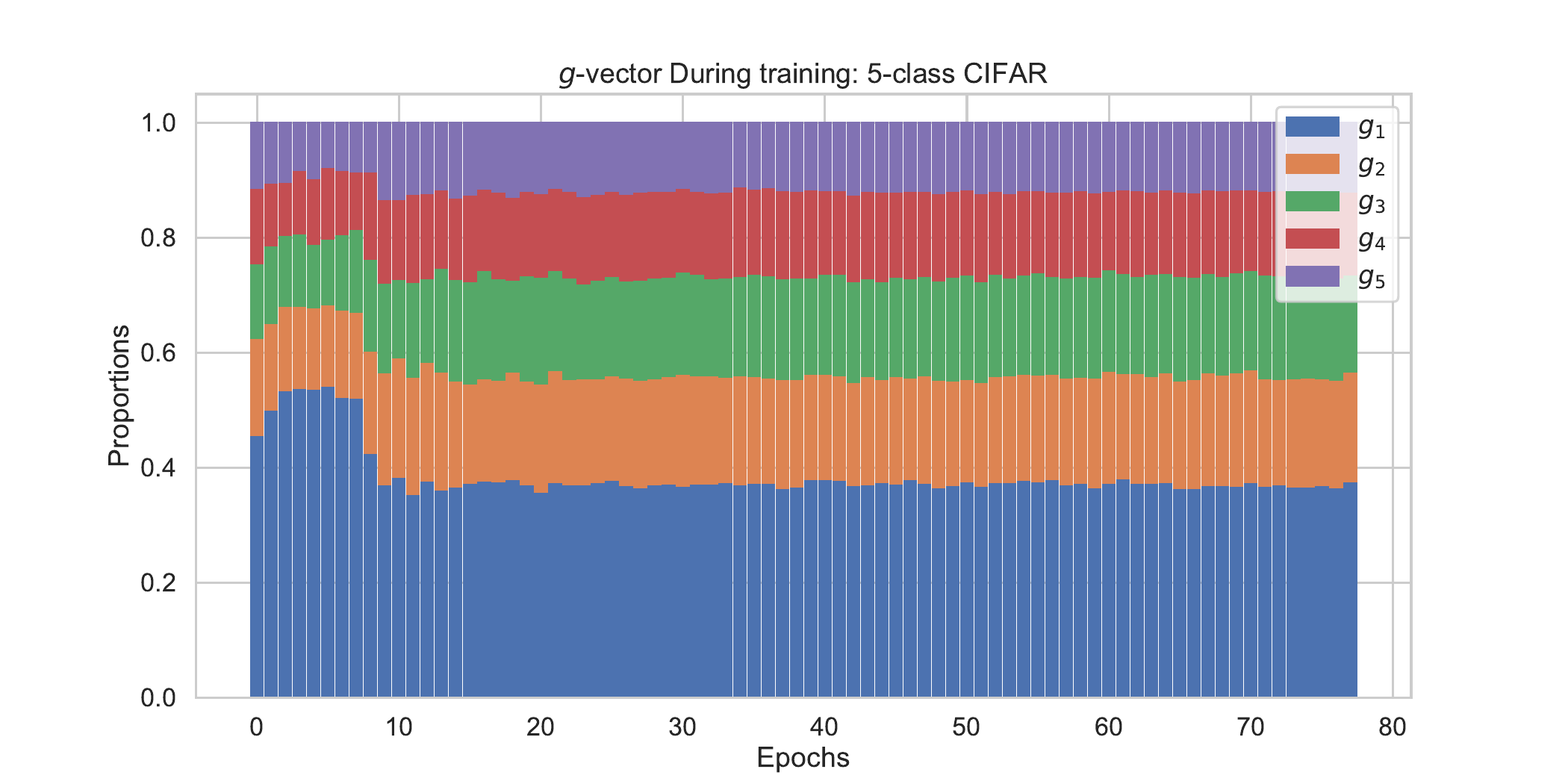}
  \caption{We train a neural network classifier on an asymmetrically noised 5-class CIFAR dataset. At each epoch we evaluate on a held-out dataset of samples. Since we have access to the noisy posteriors we can compute the $g-$vectors for the classifier at each epoch. We plot these as a stacked bar chart. 
}
  \label{ch3:fig:stacked_bars5}
  \includegraphics[width=0.9\textwidth]{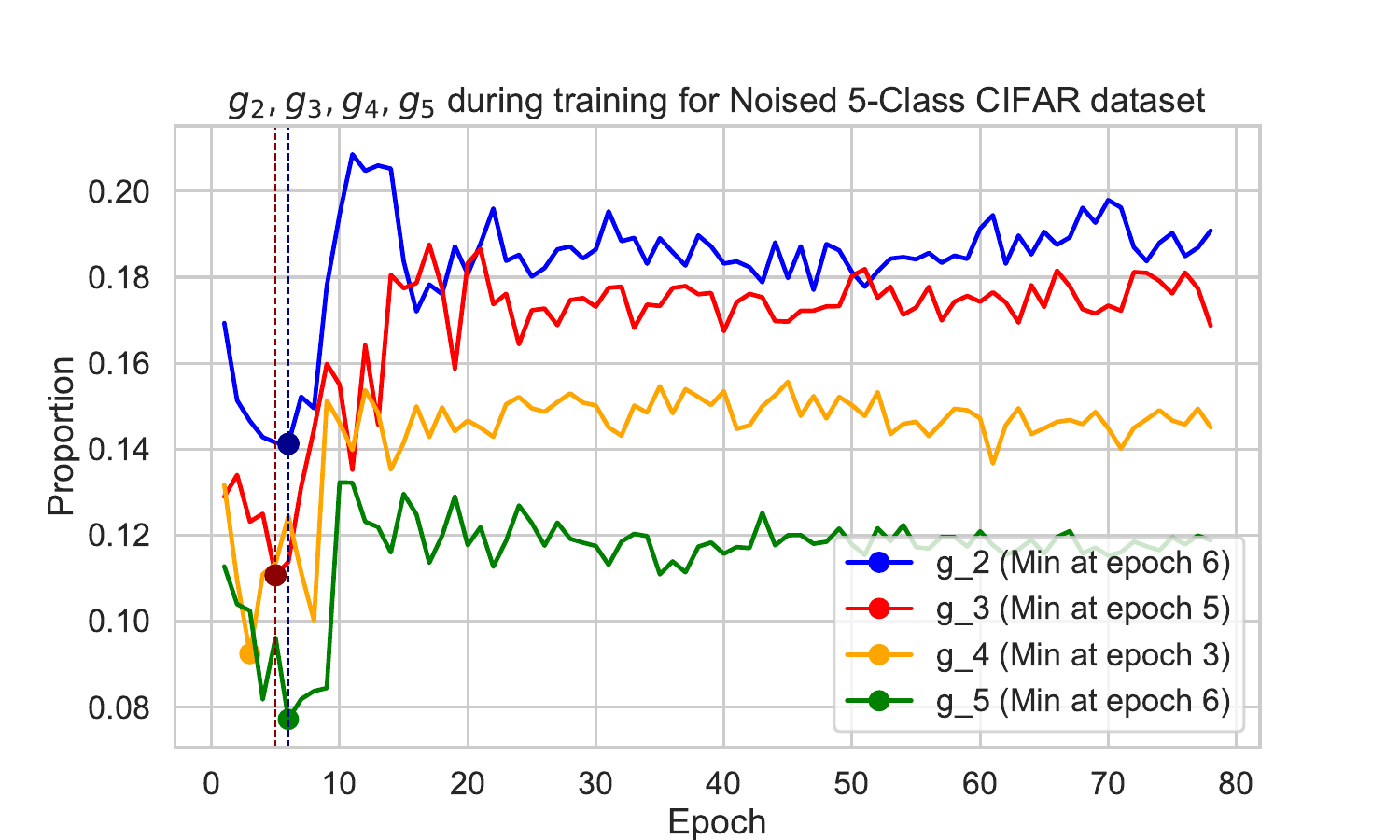}
  \caption{As in Figure~\ref{ch3:fig:simul} we train a neural network classifier on an asymmetrically noised 5-class CIFAR dataset. At each epoch we evaluate on a held-out dataset of samples. Since we have access to the noisy posteriors we can compute the $g-$vectors for the classifier at each epoch. We plot $g_2$ (blue), $g_3$ (red), $g_4$ (yellow) and $g_5$ (green) during training and note when each attains its minima. Each minima occurs almost simultaneously. 
}
  \label{ch3:fig:simul5}
\end{figure*}

\end{document}